\documentclass[11pt,a4paper]{article}
\usepackage[margin=27mm]{geometry}
\usepackage[T1]{fontenc}
\usepackage{lmodern}
\usepackage{amsmath,amssymb,amsthm,mathtools}
\usepackage{booktabs,array,graphicx}
\usepackage{microtype}
\usepackage[section]{placeins}

\usepackage[hidelinks]{hyperref}
\newtheorem{proposition}{Proposition}
\newtheorem{lemma}{Lemma}
\theoremstyle{definition}
\newtheorem{definition}{Definition}
\newtheorem{example}{Example}
\theoremstyle{remark}

\newcommand{\R}{\mathbb R}
\newcommand{\tr}{\operatorname{tr}}
\newcommand{\col}{\operatorname{col}}
\newcommand{\nul}{\operatorname{null}}
\newcommand{\Var}{\operatorname{Var}}

\title{What Can a Gaussian Process Design Test?\\[4pt]
\large Structural and Prior-Based Detection from a Gale Perspective}
\author{Ivan De Boi\\Department of Electromechanical Engineering\\University of Antwerp\\Campus Groenenborger -- Z373\\Groenenborgerlaan 171, 2020 Antwerp, Belgium\\\texttt{ivan.deboi@uantwerpen.be} \and Marnix Van Soom\\Artificial Intelligence Lab\\Vrije Universiteit Brussel\\Pleinlaan 2, 1050 Brussels, Belgium\\\texttt{marnix@ai.vub.ac.be}}
\date{}

\begin{document}
\maketitle
\begin{abstract}
A Gaussian process (GP) model can agree with the data for two reasons: its assumptions are right, or the chosen inputs could never have shown that they are wrong. The distinction can be checked from the design before any responses are observed. Every model implies relations that its noiseless responses must satisfy at the chosen inputs, such as ``the middle value lies on the line through its two neighbours''. For GPs built from finitely many features, these relations are exactly the null space of the kernel matrix. Gale duality gives them a geometric interpretation, in which each observation has a vector and the smallest groups of observations that can expose an error are the circuits. For other kernels the relations become soft: response patterns may be improbable under the prior rather than algebraically impossible. A standard test then combines two kinds of evidence. Structural evidence comes from a violated relation and grows without limit as the noise falls. Prior-based evidence only says that a departure is improbable under the prior. With all inputs at the two ends of an interval, for example, a GP can reject a straight line against a large curvature, but only because the implied intercept is improbable, never because curvature was seen. In simulations the predicted power matched the observed rejection rates. Choosing the next input by predicted power raised the power against a localised discrepancy from $0.48$ to $0.72$, against $0.51$ when choosing by predictive variance, and a grid in two dimensions contained exact tests of additivity that a Latin hypercube lacked. The test itself is classical. The contribution is the prospective reading of that test: before observing the responses, the design already determines what kind of contradiction it can produce.
\end{abstract}

\section{Introduction}
Gaussian processes (GPs) are widely used as surrogate models when data are expensive, most prominently in Bayesian optimisation and active learning \cite{rasmussen2006,shahriari2016,garnett2023}, with applications ranging from the calibration of measurement systems \cite{deboi2022} to clinical assessment \cite{deboi2024} and the tuning of accelerator beamlines \cite{ramosgarces2024}. Their appeal rests on calibrated uncertainty, but that calibration is conditional on the model. A GP model can pass the checks applied to the collected data for two very different reasons: because its assumptions are adequate, or because the data could never have contradicted them. Consider a straight-line model and a design that places all observations at the two ends of an interval. A straight line can always be drawn through the mean responses at two inputs, so the data agree with the model whatever the true curvature is, and no amount of further data at the same inputs can change this. Agreement with the data is then no evidence that the function is straight.

Diagnosing a model therefore involves two different questions, which this paper keeps apart. The first is \emph{detection}: after the data are collected, are some observations incompatible with the model? Classical regression diagnostics and their GP counterparts answer this question \cite{cook1977,belsley1980,cook1982,bastos2009}. The second is \emph{testability}: before the data are collected, which violations of the model could the design reveal at all, and with what probability? For linear models and a specified alternative, designs for detecting model inadequacy \cite{box1959,atkinson1972,jones1978} and for discriminating between rival models \cite{atkinson1975,dette2009} answer this question. For GPs, designs that discriminate between specified rival models exist \cite{gardner2015,olofsson2018}. This work makes explicit, for a single GP model, which departures a given design can detect and with what power. Such an account must handle correlated observations and kernels that need not correspond to a finite set of features. This matters because the designs used with GPs are usually chosen for other purposes. Observations placed where predictive uncertainty is largest need not be the ones that would most strongly challenge the model assumptions.

Figure~\ref{fig:teaser} shows both points on simple examples. In panel (a), a straight-line model is tested against a curvature $cx^2$ with $40$ observations. When all inputs lie at the two ends of an interval, the residual test stays at its level for every curvature. A GP test would eventually reject a very large curvature at that design, but only because the intercept it implies is improbable under the prior, not because any relation is violated (Section~\ref{sec:tests}). Moving four of the observations to the centre makes the same curvature structurally detectable. In panel (b), the next observation that would best test a GP against a narrow bump lies at the bump, where the predictive variance is small, and not at the boundary, where it is largest. Both answers are available before any response is measured.

\begin{figure}[!htbp]
\centering
\includegraphics[width=\linewidth]{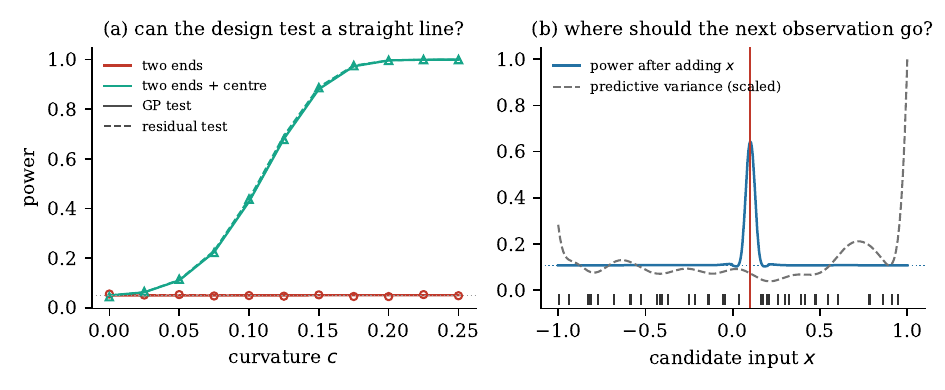}
\caption{What a design can test, computed before any data are collected (details in Section~\ref{sec:experiments}). (a) Power of two tests of a straight-line model against a curvature $cx^2$, for $40$ inputs at the two ends of $[-1,1]$, and for $36$ inputs at the ends with $4$ at the centre. Lines show predicted power, markers simulated rejection rates. The end-point design contains no consistency relation that involves curvature. (b) Power of a GP test against a narrow bump at $x_0=0.1$ (vertical line) after adding one observation at $x$ (solid), and the current predictive variance (dashed, scaled). Ticks mark the current inputs, and the dotted line the current power.}\label{fig:teaser}
\end{figure}

A model can only produce certain response patterns. Every linear relation shared by all of those patterns gives a check on the data. We call such a relation a \emph{consistency relation}. If a function is assumed to be a straight line, for example, then at three equally spaced inputs the middle value must equal the average of its neighbours. If the model is correct, a consistency relation can fail only by as much as the noise explains. At the two-ended design above, no consistency relation involves curvature, and that is why the design cannot test it.

For GPs built from finitely many features, the consistency relations are exactly the null space of the kernel matrix, the combinations of function values that the prior fixes at zero. We describe these relations geometrically with Gale duality \cite{gale1956,ziegler1995,deloera2010}. It gives each observation a vector, its \emph{Gale vector}, that describes its role in the relations, and it singles out the \emph{circuits}, the smallest groups of observations that can expose a structural error. For other kernels, such as the radial basis function (RBF) kernel, no exact relations exist. Some patterns are then merely unlikely, and we use the spectrum of the kernel matrix to describe these soft restrictions. Gale duality is the language for the exact case, not the source of the algebra.

The relations span the same residual space as classical diagnostics. Our interest is therefore not in constructing another retrospective residual test, but in using that space prospectively. Before observing $y$, the model and the chosen inputs already determine which departures can leave a residual and which can be absorbed by the model. This makes testability a property of the design itself.

For a finite-rank GP, the exact relations are the null space of the kernel matrix. Gale duality provides a geometric representation of these relations and identifies their minimal supports through circuits. For a full-rank GP, no exact relations remain, but the kernel spectrum still separates strongly penalised response directions from flexible ones. The corresponding power calculation distinguishes evidence from a violated relation from evidence that is due only to prior improbability.

The residual projector, multiple-case diagnostic identities, noncentral chi-square power calculations and classical discrimination criteria used below are not new. The contribution is the way they are put together to answer a design question: what could this GP design actually contradict, and what kind of evidence would a rejection represent? Because this question can be answered before the responses are measured, it also gives a prospective criterion for choosing observations. The analysis assumes a fixed mean, kernel and noise variance; estimation from the same data and repeated testing are treated as limitations rather than folded into the exact theory.

The same distinction is already latent in the marginal likelihood. If $K=\sum_j\lambda_ju_ju_j^\top$, then
$y^\top(K+\sigma^2I)^{-1}y=\sum_j(u_j^\top y)^2/(\lambda_j+\sigma^2)$.
Our use of this decomposition is to attach those directions to exact or soft relations of the design and to compute, before measurement, the power against a specified departure. The experiments then ask whether that prospective calculation agrees with simulation and whether it selects different observations from predictive variance.
\section{Related work}\label{sec:related}

Influence and outlier diagnostics for linear models have a long history, including Cook's distance \cite{cook1977}, the systematic treatments of \cite{belsley1980,cook1982}, and robust alternatives \cite{rousseeuw1987,atkinson2000}. A well-known difficulty is \emph{masking}: several outliers can conceal one another so that each appears innocuous when assessed individually \cite{beckman1983}. Group and subset-based procedures were developed precisely for this situation \cite{hawkins1984,hadi1993}. Our group tests are GP analogues of these multiple-case diagnostics.

Closed-form leave-one-out (LOO) residuals for kriging and GP regression follow from the inverse covariance matrix \cite{dubrule1983,sundararajan2001,rasmussen2006}. Multiple-fold residuals and their covariances admit analogous block formulae \cite{ginsbourger2025}. Diagnostics for GP emulators of computer models, including Mahalanobis-type statistics and pivoted Cholesky errors, were systematised in \cite{bastos2009}. The relative behaviour of maximum likelihood and cross-validation for hyperparameter estimation under misspecification is studied in \cite{bachoc2013}. Approximate LOO for non-Gaussian likelihoods is reviewed in \cite{vehtari2016}. The group tests of Section~\ref{sec:tests} are of this classical type. Here they are used prospectively: their power is computed from the design and interpreted through the model's consistency relations.

Instead of testing a fixed observation model, one can make it more flexible: Student-$t$ likelihoods \cite{jylanki2011}, input-dependent noise \cite{goldberg1998,kersting2007}, or data-point-specific noise variances selected by sequential relevance pursuit \cite{ament2024}. These approaches repair the model rather than diagnose it, and are natural alternatives once a test has identified an incompatibility.

In the calibration of computer models, systematic mismatch between a model and reality is represented by an explicit discrepancy term \cite{kennedy2001}, and ignoring it can lead to biased and overconfident inference \cite{brynjarsdottir2014}. In Bayesian inference, model criticism through predictive checks goes back to \cite{box1980} and was developed further in \cite{gelman1996}. Our tests are frequentist checks of the hypothesis that the data come from a fully specified GP (Section~\ref{sec:tests}), but they share the aim of separating criticism from estimation.

Designs that detect the inadequacy of a regression model go back to the treatment of bias from omitted terms in \cite{box1959}, and were developed into explicit design criteria in \cite{atkinson1972,jones1978}. Choosing experiments to discriminate between rival models has a similarly long history, from the entropy-based sequential designs of \cite{box1967} and the T-optimality criterion of \cite{atkinson1975} to its later analysis and extensions \cite{dette2009,lopezfidalgo2007,atkinson2007}. In machine learning, active model selection chooses observations to discriminate between GP models \cite{gardner2015}, and GP surrogates make discrimination designs available for non-analytical models \cite{olofsson2018}. Section~\ref{sec:discrimination} positions our power results relative to this literature.

Gale duality originates in the study of convex polytopes \cite{gale1956} and is treated in standard texts \cite{grunbaum2003,ziegler1995,matousek2002,deloera2010}. Algebraic statistics describes the identifiable polynomial models of a design through its design ideal \cite{pistone1996,pistone2001}, and circuits have been used to generate randomisation schemes \cite{pesce2023}. Classical geometric structure has been combined with GP models before, for example line geometry for surface approximation with GP latent variable models \cite{deboi2023}. Here Gale duality serves as the language that makes the relations implied by a GP design, and their effect on the power of its tests, explicit.

\section{Gale duality and model consistency relations}\label{sec:fit}
Consider first a linear regression model with a finite set of features. Its feature matrix spans the response vectors that the model can represent exactly. The orthogonal complement contains the linear relations that every noiseless response must satisfy. Gale duality organises those relations observation by observation and identifies their smallest supports.

\subsection{Gale duality}
We begin with the classical construction. Let $p_1,\ldots,p_n\in\R^d$ be $n$ points whose affine hull is all of $\R^d$. An \emph{affine dependence} of the configuration is a vector $\gamma\in\R^n$ with
\begin{equation}\label{eq:affdep}
 \sum_{i=1}^n\gamma_i=0,\qquad\sum_{i=1}^n\gamma_ip_i=0.
\end{equation}
Informally, an affine dependence is a way of weighting the points so that both the weights and the weighted positions cancel. The affine dependences form a linear subspace of $\R^n$ of dimension $n-d-1$. They matter for data because an affine dependence annihilates the values of every affine function. If $y_i=a+b^\top p_i$, then $\sum_i\gamma_iy_i=a\sum_i\gamma_i+b^\top\sum_i\gamma_ip_i=0$ by \eqref{eq:affdep}. Example~\ref{ex:three} works this out for three points on a line.

Choose a basis of this subspace and write it as the columns of a matrix $G\in\R^{n\times(n-d-1)}$. The $n$ rows $g_1^\top,\ldots,g_n^\top$ of $G$, with $g_i\in\R^{n-d-1}$, form the \emph{Gale transform} of the configuration: a second configuration of $n$ vectors, one per original point, in a space of complementary dimension \cite{gale1956,grunbaum2003,ziegler1995,matousek2002,thomas2006}. The primal configuration records where the points are. The dual configuration records how they depend on each other. The power of the construction is that properties of the one can be read off from the other. In this paper the property read off is which sets of observations witness which relations (Proposition~\ref{prop:circuits}).

The same construction works far beyond points in space. Conditions \eqref{eq:affdep} say that $\gamma$ is orthogonal to every column of the $n\times(d+1)$ matrix with rows $(1,p_i^\top)$, the vectors that describe the points. Replacing these rows by any feature vectors gives the version we use throughout.

We write $\mathcal X$ for the input space and $x_1,\ldots,x_n\in\mathcal X$ for the $n$ inputs, which together form the \emph{design}. A feature map $\phi:\mathcal X\to\R^m$ turns each input into $m$ features, and the feature matrix $\Phi\in\R^{n\times m}$ has one row $\phi(x_i)^\top$ per observation. Its rank is $r=\operatorname{rank}(\Phi)$. The relations live in $\nul(\Phi^\top)=\{v\in\R^n:\Phi^\top v=0\}$, a subspace of dimension $n-r$.
\begin{definition}[Gale representation of a feature configuration]\label{def:gale}
A matrix $G\in\R^{n\times(n-r)}$ whose columns form a basis of $\nul(\Phi^\top)$ is a Gale representation of the configuration of rows of $\Phi$. Its $i$th row $g_i^\top$, with $g_i\in\R^{n-r}$, is the Gale vector of observation $i$.
\end{definition}
For the affine feature map $\phi(x)=(1,x^\top)^\top$ with $\mathcal X=\R^d$, Definition~\ref{def:gale} recovers the classical Gale transform.
\begin{example}[Three points on a line]\label{ex:three}
Take the three points $0,1,2$ of the real line ($n=3$, $d=1$) with the affine feature map, so that
\begin{equation*}
 \Phi=\begin{pmatrix}1&0\\1&1\\1&2\end{pmatrix},\qquad
 G=\frac1{\sqrt6}\begin{pmatrix}1\\-2\\1\end{pmatrix},
\end{equation*}
with $r=2$ and $n-r=1$. One checks that $\Phi^\top G=0$. Every affine dependence is a multiple of $(1,-2,1)$, since $1-2+1=0$ and $1\cdot0-2\cdot1+1\cdot2=0$, and this is the straight-line relation of the introduction. The values $1,3,5$ that a straight line takes at the three points satisfy $1-2\cdot3+5=0$, whereas the values $1,4,5$, which no straight line passes through, give $1-2\cdot4+5=-2$. The Gale vectors are the numbers $1/\sqrt6$, $-2/\sqrt6$ and $1/\sqrt6$: the middle point plays the opposite role to its two neighbours.
\end{example}
A Gale representation is not unique. Another basis of $\nul(\Phi^\top)$ replaces $G$ by $GM$ for an invertible $M\in\R^{(n-r)\times(n-r)}$, which changes the Gale vectors but not the relations they encode. Unless stated otherwise we therefore take $G$ with orthonormal columns, $G^\top G=I_{n-r}$, where $I_k$ is the $k\times k$ identity matrix and $I$ without subscript is $n\times n$. Only rotations then remain free, so the lengths of the Gale vectors and the angles between them are meaningful.

Classically, the information of interest is combinatorial and is read off from the signs of the Gale vectors. In this paper we mainly use their lengths and inner products, because these determine the leverage of an observation (Section~\ref{sec:leverage}) and the power of the tests (Section~\ref{sec:power}).

\subsection{Consistency relations of a model}
Now attach a response $y_i\in\R$ to each input, and collect them in $y\in\R^n$. Consider the linear-in-parameters model
\begin{equation}\label{eq:linear}
 y=\Phi w+\varepsilon,\qquad \varepsilon\sim\mathcal N(0,\sigma^2I),
\end{equation}
with parameter vector $w\in\R^m$ and observation-noise variance $\sigma^2>0$. The column space $\col(\Phi)\subseteq\R^n$ contains every noiseless response vector that the model can represent. Each column $\gamma\in\R^n$ of a Gale representation satisfies $\gamma^\top\Phi=0$, and hence $\gamma^\top\Phi w=0$ for every $w$. Every exactly representable response therefore satisfies $\gamma^\top y=0$. We call such a linear constraint a \emph{consistency relation} of the model at the design. The vector
\begin{equation}
 G^\top y=\sum_{i=1}^n y_i\,g_i\in\R^{n-r}
\end{equation}
collects the values of a complete, non-redundant set of consistency relations. It vanishes for every response the model can represent exactly, and its departures from zero are attributable either to noise or to a violation of the model. In this picture, observation $i$ contributes its Gale vector $g_i$, weighted by its response, to the total failure of the relations. The Gale vectors are fixed by the model and the design before any response is observed. They determine how a departure at a given observation can show up in the tests.

\subsection{Residual projector, leverage and noise}\label{sec:leverage}
The Gale representation is closely tied to classical regression diagnostics. When $\Phi$ has full column rank, the hat matrix $H=\Phi(\Phi^\top\Phi)^{-1}\Phi^\top\in\R^{n\times n}$ projects orthogonally onto $\col(\Phi)$, and the residual projector is $P_N=I-H$ \cite{cook1982}. To avoid a full-rank assumption, for example when there are more features than distinct inputs, we write more generally
\begin{equation}\label{eq:projector}
 P_N=I-\Phi\Phi^+=GG^\top\in\R^{n\times n},
\end{equation}
where $\Phi^+\in\R^{m\times n}$ is the Moore--Penrose pseudo-inverse. The matrix $\Phi\Phi^+$ is the orthogonal projector onto $\col(\Phi)$ for any $\Phi$, and it coincides with $H$ in the full-rank case. The second equality holds because both sides are the orthogonal projector onto $\nul(\Phi^\top)$. The vector $P_Ny=G(G^\top y)$ is the ordinary residual vector, so the Gale coordinates $G^\top y$ are the residuals expressed in an orthonormal basis of the residual space. Under \eqref{eq:linear},
\begin{equation}\label{eq:chi2basic}
 G^\top y=G^\top\varepsilon\sim\mathcal N(0,\sigma^2I_{n-r}),\qquad
 \frac{\|G^\top y\|^2}{\sigma^2}=\frac{\|P_Ny\|^2}{\sigma^2}\sim\chi^2_{n-r}.
\end{equation}
This is the familiar distribution of the residual sum of squares. The distributional statement requires the model and noise assumptions, not merely the existence of a null space: the null space tells us \emph{which} relations to examine, and the noise model tells us \emph{how large} their failures may be by chance.

The diagonal of \eqref{eq:projector} gives
\begin{equation}\label{eq:leverage}
 \|g_i\|^2=(P_N)_{ii}=1-h_{ii},
\end{equation}
where $h_{ii}=H_{ii}$ is the leverage of observation $i$ \cite{belsley1980}, and the off-diagonal entries give $g_i^\top g_j=(P_N)_{ij}$. A point with large leverage has a short Gale vector and hence a small footprint in the residual space: the fitted model can absorb most of a perturbation at that point by adjusting its coefficients. Conversely, a long Gale vector means that a perturbation at that point is largely visible in the residuals. Leverage is thus a statement about the \emph{observability} of a local departure, and this is how we shall use it. Example~\ref{ex:square} illustrates these quantities.

\begin{example}[A planar model on the unit square]\label{ex:square}
Take the four corners $(0,0),(1,0),(1,1),(0,1)$ of the unit square, numbered $1$ to $4$, and the centre $(1/2,1/2)$, numbered $5$, with the planar model $\phi(x)=(1,x_1,x_2)$. Here $n=5$, $m=r=3$, and $\Phi\in\R^{5\times3}$, so there are $n-r=2$ independent consistency relations. A convenient basis is
\begin{equation}
 y_1-y_2+y_3-y_4=0,\qquad
 y_1+y_2+y_3+y_4-4y_5=0.
\end{equation}
The first states that a plane has zero mixed difference, or `twist', over the square: the change in one coordinate direction is the same on both sides of the square. The second states that the value of a plane at the centre is the average of its values at the corners. The two coefficient vectors are orthogonal, and normalising them gives $G\in\R^{5\times2}$ with Gale vectors
\begin{equation}
 g_1=g_3=\bigl(\tfrac12,\tfrac1{\sqrt{20}}\bigr),\quad
 g_2=g_4=\bigl(-\tfrac12,\tfrac1{\sqrt{20}}\bigr),\quad
 g_5=\bigl(0,-\tfrac4{\sqrt{20}}\bigr),
\end{equation}
shown in Figure~\ref{fig:square}. Hence $(P_N)_{ii}=0.3$ for each corner and $(P_N)_{55}=0.8$, that is, leverages $0.7$ and $0.2$. A perturbation at the centre violates only the second relation, and $80\%$ of its squared size appears in the residuals. A perturbation at a corner violates both relations, but only $30\%$ of its squared size remains visible, because a tilted plane can absorb most of it. Opposite corners have identical Gale vectors, so no test can tell an offset at one from an offset at the other. The Gale vectors sum to zero only because the model contains an intercept.
\end{example}

\begin{figure}[!htbp]
\centering
\includegraphics[width=0.9\linewidth]{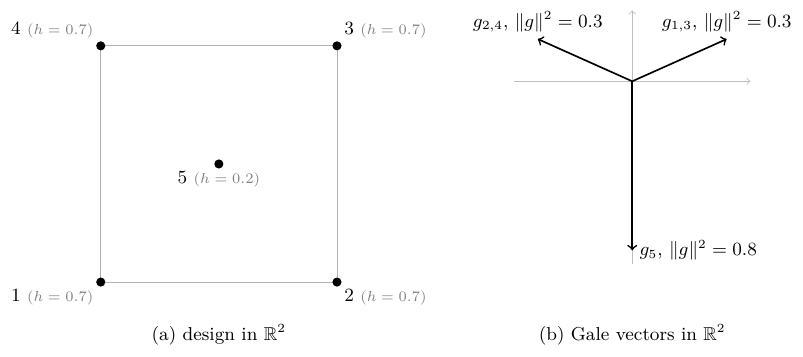}
\caption{Example~\ref{ex:square}. (a) The design, five inputs in $\R^2$, with the leverage $h_{ii}$ of each observation under the planar model. (b) The corresponding Gale vectors $g_i\in\R^{n-r}=\R^2$. Opposite corners share a Gale vector. The centre has the longest Gale vector and the smallest leverage, so departures there are the most visible in the residuals.}\label{fig:square}
\end{figure}

Among all consistency relations, those with minimal support play a special role. A non-zero $\gamma\in\nul(\Phi^\top)$ is a \emph{circuit} of the configuration if no other non-zero element of $\nul(\Phi^\top)$ has a support strictly contained in that of $\gamma$ \cite{ziegler1995,bjorner1999}. Circuits are the most local tests a design offers. Each involves as few observations as possible, so its failure involves only those observations, although it does not identify which of them is at fault. The relation $(1,-2,1)$ of Example~\ref{ex:three} is a circuit. Up to scaling, Example~\ref{ex:square} has three circuits: the twist relation and the two diagonal relations $y_1+y_3-2y_5=0$ and $y_2+y_4-2y_5=0$, which state that the value at the centre is the average of each pair of opposite corners. The centre relation of Example~\ref{ex:square} is the sum of the two diagonal relations and is therefore not a circuit. Proposition~\ref{prop:circuits} collects standard facts of matroid theory \cite{bjorner1999}, stated in the form we use for experimental design.
\begin{proposition}[Circuits are minimal witnesses]\label{prop:circuits}
For a set $C\subseteq\{1,\ldots,n\}$ of observations, write $\Phi_C$ for the rows of $\Phi$ indexed by $C$.
\begin{enumerate}
\item[(i)] $C$ supports a non-zero consistency relation exactly when the rows of $\Phi_C$ are linearly dependent, that is, when $\operatorname{rank}(\Phi_C)<|C|$.
\item[(ii)] The supports of the circuits are the minimal sets $C$ with this property.
\item[(iii)] The circuits span $\nul(\Phi^\top)$. A departure $\delta$ therefore violates some consistency relation, $G^\top\delta\neq0$, exactly when it violates some circuit.
\end{enumerate}
\end{proposition}
\begin{proof}
A relation supported on $C$ is a non-zero $\gamma$ with $\gamma_i=0$ for $i\notin C$ and $\gamma_C^\top\Phi_C=0$, which exists exactly when the rows of $\Phi_C$ are dependent. Part (ii) is the definition of a circuit. For (iii), let $\gamma\neq0$ be a relation. Among the relations whose support lies in that of $\gamma$, one of minimal support is a circuit $\gamma'$. For a coordinate $i$ in the support of $\gamma'$, the relation $\gamma-(\gamma_i/\gamma'_i)\gamma'$ has a strictly smaller support than $\gamma$, and induction on the size of the support writes $\gamma$ as a combination of circuits. If every circuit annihilated $\delta$, so would every relation.
\end{proof}
A circuit support is thus the smallest set of measurements on which the model imposes a non-trivial exact constraint. Every proper subset of it permits arbitrary noiseless responses, whereas the responses on the full circuit must satisfy one consistency relation. By (iii), every structural contradiction that a design can produce is witnessed by such a minimal set. In Example~\ref{ex:square} the twist is witnessed by the four corners alone, and the diagonal relations by three observations each. For polynomials of degree $q$ in one dimension, any $q+2$ distinct inputs form a circuit, which is the minimal design of Example~\ref{ex:veronese}. This is the role of Gale duality in this paper: it organises the exact relations through their minimal supports, making the local structure of structural testability explicit.

\subsection{Structural assumptions from feature maps}
Definition~\ref{def:gale} applies to any feature map, and the feature map determines which structural assumption the relations encode. For polynomial models the relevant feature map is the Veronese map. In one dimension a polynomial of degree at most $q$ is a linear function of
\begin{equation}\label{eq:veronese}
 \phi_q(x)=(1,x,x^2,\ldots,x^q)^\top\in\R^{q+1},
\end{equation}
and in $d$ dimensions of the vector $\phi_q(x)$ of all $\binom{d+q}{q}$ monomials of degree at most $q$, which is the affine form of the Veronese map of algebraic geometry \cite{michalek2021}. The Gale representation of the configuration $\phi_q(x_1),\ldots,\phi_q(x_n)$ therefore contains exactly the linear relations that every polynomial of degree at most $q$ satisfies at the design. A departure of higher degree is structurally visible exactly when it violates one of them, that is, when it is not representable by a polynomial of degree $q$ at the design. Example~\ref{ex:veronese} works this out for the question that runs through the paper.

\begin{example}[A straight line against curvature]\label{ex:veronese}
Take the straight-line model, $\phi_1(x)=(1,x)^\top$, and a curvature $\delta_i=x_i^2$. Suppose that the design has $k$ distinct inputs, each possibly repeated. Then $r=\operatorname{rank}(\Phi)=\min(k,2)$, and there are $n-r$ consistency relations.
\begin{itemize}
\item \emph{Two distinct inputs $a\neq b$.} Every relation compares replicates, $y_i-y_j=0$ for $x_i=x_j$. Any function takes equal values at equal inputs, so no relation involves curvature. Indeed, the vector $(x_i^2)$ equals the straight line $a^2+(a+b)(x-a)$ at both inputs, so it lies in $\col(\Phi)$ and $G^\top\delta=0$. No number of observations at the two inputs can reveal curvature.
\item \emph{Three distinct inputs $a<b<c$.} Besides the replicate comparisons there is one relation between the responses at the three inputs, $(c-b)\,y(a)-(c-a)\,y(b)+(b-a)\,y(c)=0$, which for $a,b,c=0,1,2$ is the relation $(1,-2,1)$ of Example~\ref{ex:three}. On the curvature it takes the value $(c-b)a^2-(c-a)b^2+(b-a)c^2=(b-a)(c-b)(c-a)>0$, so curvature always violates it, and more strongly the more the inputs are spread out.
\end{itemize}
The same argument settles degree $q$ against degree $q+1$. If the design has at most $q+1$ distinct inputs, a polynomial of degree $q$ can interpolate $x^{q+1}$ at all of them, so the departure violates no relation. If it has $q+2$ distinct inputs, no such polynomial exists, since $x^{q+1}-p(x)$ would then be a non-zero polynomial of degree $q+1$ with $q+2$ roots. A design can therefore test degree $q$ against degree $q+1$ exactly when it has at least $q+2$ distinct inputs. In several dimensions the number of distinct inputs is no longer enough, and their position matters. The twist $x_1x_2$, for example, cannot be tested against a plane on any design whose inputs all lie on one line parallel to an axis, since there it equals an affine function. Which departures a design can test is then a question about the Veronese configuration of the inputs, and Section~\ref{sec:structured} answers it with the Gale representation. Which polynomial terms a design can identify, and which are aliased with one another, is the subject of the algebraic theory of designs, in which the design is described by a polynomial ideal \cite{pistone1996,pistone2001}. This example is the simplest case of that theory. Our use of it is to carry the relations over to GP kernels and to their power.
\end{example}

Other structural hypotheses imply relations of the same kind.
\begin{example}[Periodicity and additivity]\label{ex:structures}
A model with known period $\tau$ implies $f(x)-f(x+\tau)=0$ for the noiseless values whenever both inputs are observed. An additive model $f(x_1,x_2)=f_1(x_1)+f_2(x_2)$ implies, on any $2\times2$ grid of inputs $\{a,a'\}\times\{b,b'\}$, the mixed difference
\begin{equation}\label{eq:mixed}
 f(a,b)-f(a,b')-f(a',b)+f(a',b')=0,
\end{equation}
and the twist relation of Example~\ref{ex:square} is the special case in which both $f_1$ and $f_2$ are linear.
\end{example}
Changing the feature map therefore changes the assumptions being tested. A polynomial lift is appropriate for a polynomial hypothesis. It does not automatically describe the assumptions of a GP with a radial basis function (RBF) kernel. For this reason, Section~\ref{sec:gp} constructs the relations from the kernel itself.

\section{Consistency tests for Gaussian processes}\label{sec:gp}
The same construction can be read directly from a GP kernel. For a finite-rank kernel, the null space of the kernel matrix is exactly the space of consistency relations. A strictly positive-definite kernel has no non-trivial null space at distinct inputs; in that case the relevant information is spectral rather than algebraic.

A Gaussian process (GP) is a probability distribution over functions $f:\mathcal X\to\R$ such that, for any finite set of inputs, the vector of function values is jointly Gaussian \cite{rasmussen2006}. A zero-mean GP is determined by its covariance function, or kernel, $k(x,x')=\operatorname{Cov}(f(x),f(x'))$, and we write $f\sim\mathcal{GP}(0,k)$. The linear model of Section~\ref{sec:fit} is a special case. If the parameters in \eqref{eq:linear} receive a prior $w\sim\mathcal N(0,\Sigma_w)$, with a positive definite $\Sigma_w\in\R^{m\times m}$, then $f(x)=\phi(x)^\top w$ is a (degenerate) GP with kernel $k(x,x')=\phi(x)^\top\Sigma_w\phi(x')$. This kernel has finite rank, because it is built from $m$ features. Kernels such as the RBF kernel $k(x,x')=s\exp(-\|x-x'\|^2/2\ell^2)$, with lengthscale $\ell>0$ and output scale $s>0$, correspond to infinitely many features and encode smoothness rather than a finite-dimensional functional form \cite{rasmussen2006}. Conditioning a GP on noisy observations gives a posterior mean and variance at every input in closed form, which is what makes GPs attractive as surrogate models. We now ask: which consistency relations does such a model imply at a design?

Model the responses as $y_i=f(x_i)+\varepsilon_i$, with a latent function $f\sim\mathcal{GP}(0,k)$ with kernel $k:\mathcal X\times\mathcal X\to\R$, and independent noise $\varepsilon_i\sim\mathcal N(0,\sigma^2)$. We work with a fixed zero mean, or with responses centred by a specified mean function. Let $K\in\R^{n\times n}$ be the positive-semidefinite kernel matrix, $K_{ij}=k(x_i,x_j)$. Two further matrices appear throughout,
\begin{equation}\label{eq:gpnotation}
 A=K+\sigma^2I,\qquad Q=A^{-1},
\end{equation}
both in $\R^{n\times n}$. The matrix $A$ is the covariance of the responses $y$, and $Q$ is its inverse, the precision matrix. Both depend on $\sigma^2$, and we write $A(\sigma^2)$ or $Q(\sigma^2)$ only where that matters.

Unless stated otherwise, the kernel hyperparameters and the mean are fixed. Fitting them to the same observations changes the calibration of the tests, a point we return to in Section~\ref{sec:tests}.

\subsection{Finite-rank kernels: exact relations}

For a finite-rank kernel $K=\Phi\Phi^\top$ with $\Phi\in\R^{n\times m}$,
\begin{equation}
\nul(K)=\nul(\Phi^\top),
\label{eq:null-kernel-feature}
\end{equation}
since for any $v\in\R^n$,
\[
v^\top Kv=\|\Phi^\top v\|^2.
\]
The same identity holds for $K=\Phi\Sigma_w\Phi^\top$ whenever
$\Sigma_w\in\R^{m\times m}$ is positive definite, as in the GP obtained
from the Bayesian linear model \eqref{eq:linear} with prior
$w\sim\mathcal N(0,\Sigma_w)$. The exact consistency relations can
therefore be recovered directly from $K$, even when the feature map
$\Phi$ is unknown or has no convenient form.

The null space of $K$ also has a direct reading in terms of the prior. Write $f\in\R^n$ for the vector of latent values $f(x_i)$, so that $y=f+\varepsilon$. For any $v\in\R^n$, $\Var(v^\top f)=v^\top Kv$, which vanishes exactly when $Kv=0$. The null space of $K$ is therefore the set of linear combinations of function values that the prior fixes at zero with certainty, before any data are seen. For such $v$, $v^\top y=v^\top f+v^\top\varepsilon=v^\top\varepsilon$, so any systematic signal along these directions contradicts the model. The dimension $n-\operatorname{rank}(K)$ counts the (linearly) independent checks that the data can run against the prior at this design. When it is zero, the data can challenge the model only in the soft sense of Section~\ref{sec:soft}. A relation supported on a set $C$ of observations exists exactly when $v_C^\top K_{CC}v_C=0$ for some $v_C\neq0$, that is, when the kernel matrix $K_{CC}$ of those inputs is singular, so the circuits of Proposition~\ref{prop:circuits} can also be found from $K$ alone.

\begin{example}[Polynomial kernel]\label{ex:polykernel}
The polynomial kernel $k(x,x')=(1+xx')^q$ on $\R$ is the kernel of a random polynomial of degree $q$. Indeed, it factorises as $K=\Phi\Sigma_w\Phi^\top$ with $\phi=\phi_q$ and a diagonal positive definite $\Sigma_w\in\R^{(q+1)\times(q+1)}$ of binomial coefficients. It therefore has rank at most $q+1$ on any design. At $q+2$ distinct inputs its null space is spanned by a single relation, the divided difference of order $q+1$, which for $q=1$ is the relation of Example~\ref{ex:veronese}. The kernel matrix alone thus recovers the test of polynomial degree, without any feature map.

For $q=1$, $k(x,x')=1+xx'$, and at the inputs $0,1,2$ the kernel matrix is
\begin{equation}
 K=\begin{pmatrix}1&1&1\\1&2&3\\1&3&5\end{pmatrix}.
\end{equation}
One checks directly that $K(1,-2,1)^\top=0$. The prior therefore insists that $f(0)-2f(1)+f(2)=0$: the middle value is the average of the outer ones. With outer observations $1$ and $5$ the model expects a value of about $3$ in the middle. An observed $4$ gives $y_0-2y_1+y_2=-2$. Under the model this combination has variance $(1,-2,1)A(1,-2,1)^\top=6\sigma^2$, so its size relative to $\sqrt6\,\sigma$ decides whether it is plausibly noise or a failure of the straight-line assumption. The null space supplies the relation, while the data and the noise model supply the verdict.
\end{example}

For a strictly positive-definite kernel, such as the RBF or Mat\'ern kernels, the kernel matrix at distinct inputs is nonsingular and the exact null space is trivial. Every response vector is then representable. This does not mean that every response is equally plausible under the prior, and the next subsection makes this precise. Repeated inputs are the exception for every kernel. If $x_i=x_j$, rows $i$ and $j$ of $K$ coincide, so the vector with entries $+1$ at $i$, $-1$ at $j$ and $0$ elsewhere lies in $\nul(K)$. The relation $y_i-y_j=0$ then tests the noise model alone, which is the classical pure-error check for replicated observations. Inputs that are close but not equal give small eigenvalues in \eqref{eq:spectrum} instead of zero ones.

Adding kernels keeps only the relations that both kernels share. For positive-semidefinite $K_1,K_2\in\R^{n\times n}$, $v^\top(K_1+K_2)v=0$ exactly when $v^\top K_1v=v^\top K_2v=0$, so $\nul(K_1+K_2)=\nul(K_1)\cap\nul(K_2)$. This can still leave exact relations even when both kernels are flexible. For an additive kernel $k(x,x')=k_1(x_1,x_1')+k_2(x_2,x_2')$, with each $k_j$ depending on the $j$th input only, the mixed difference \eqref{eq:mixed} has zero variance under both components. It is therefore an exact relation even when $k_1$ and $k_2$ are RBF kernels, and an $n_1\times n_2$ grid of inputs carries $(n_1-1)(n_2-1)$ independent relations of this kind. Experiment E3 uses this.

\subsection{Full-rank kernels: spectral coordinates}\label{sec:soft}
Exact Gale vectors exist only when $K$ is singular. For a full-rank kernel we use instead the spectral coordinates of a residual matrix. They play the same role for lengths and inner products, but they are not an exact Gale transform, and the combinatorial reading of Section~\ref{sec:fit}, through signs and circuits, does not carry over to them.

The posterior mean of $f$ at the training inputs is $KQy$. Since $KQ=(A-\sigma^2I)Q=I-\sigma^2Q$, the corresponding residual is
\begin{equation}\label{eq:soft}
 e=y-KQy=Sy\in\R^n,\qquad S=\sigma^2Q\in\R^{n\times n}.
\end{equation}
We call $e$ the \emph{soft residual} and $S$ the \emph{residual matrix}. Let $K=U\Lambda U^\top$ be an eigendecomposition, with $U\in\R^{n\times n}$ orthogonal, columns $u_1,\ldots,u_n$, and $\Lambda=\operatorname{diag}(\lambda_1,\ldots,\lambda_n)$, $\lambda_j\geq0$. Then
\begin{equation}\label{eq:spectrum}
 S=U\operatorname{diag}\left(\frac{\sigma^2}{\lambda_j+\sigma^2}\right)U^\top.
\end{equation}
The residual matrix acts by spectral shrinkage: $Sy$ projects $y$ onto the eigenvectors of $K$, damps the coordinate along $u_j$ by the factor $\sigma^2/(\lambda_j+\sigma^2)$, and maps the result back. The spectrum therefore sorts the directions of response space into three kinds, and we use the following vocabulary throughout.
\begin{itemize}
\item $\lambda_j=0$: a \emph{hard structural relation}. The prior fixes $u_j^\top f$ at zero, and all of $u_j^\top y$ passes to the residual.
\item $0<\lambda_j\ll\sigma^2$: a \emph{soft prior restriction}. The prior lets $u_j^\top f$ vary, but by much less than the noise, so almost all of $u_j^\top y$ passes to the residual. This is a statement about plausibility, not an algebraic constraint.
\item $\lambda_j\gg\sigma^2$: a \emph{prior-flexible direction}. The residual keeps only a small fraction of $u_j^\top y$, because the model readily explains such patterns.
\end{itemize}
Figure~\ref{fig:spectrum} shows that for an RBF kernel the soft prior restrictions oscillate between neighbouring inputs and resemble higher-order finite differences, but unlike the relations of Section~\ref{sec:fit} they exclude no response pattern.

\begin{figure}[!htbp]
\centering
\includegraphics[width=\linewidth]{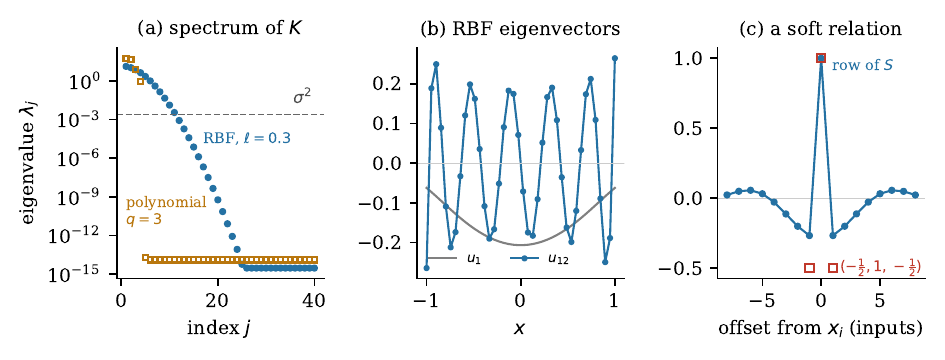}
\caption{Hard structural relations and soft prior restrictions on $n=40$ equally spaced inputs in $[-1,1]$, with noise variance $\sigma^2=2.5\times10^{-3}$. (a) Eigenvalues $\lambda_j$ of the kernel matrix for the RBF kernel with lengthscale $\ell=0.3$ and unit output scale, and for the polynomial kernel $(1+xx')^3$. Values below machine precision relative to $\lambda_1$ are shown at that level. The polynomial kernel has rank $4$, so it imposes $36$ exact relations. The RBF kernel has full rank, and the $29$ eigendirections below $\sigma^2$ are soft prior restrictions. (b) The leading eigenvector $u_1$ of the RBF kernel matrix and the first eigenvector $u_{12}$ whose eigenvalue lies below $\sigma^2$. (c) The row of the residual matrix $S$ at the central input, scaled to $1$ at that input, next to the straight-line relation $(-\frac12,1,-\frac12)$ of Example~\ref{ex:veronese}. It compares $y_i$ with a weighted average of its neighbours.}\label{fig:spectrum}
\end{figure}

Unlike the exact projector $P_N$, the residual matrix $S$ is generally not idempotent, and $\|Sy\|^2/\sigma^2$ does not in general have a chi-square distribution. Since $S$ is symmetric and positive definite, it has the symmetric square root
\begin{equation}\label{eq:spectralcoords}
 B=S^{1/2}=U\operatorname{diag}\left(\sqrt{\frac{\sigma^2}{\lambda_j+\sigma^2}}\right)U^\top\in\R^{n\times n},\qquad BB^\top=S.
\end{equation}
Its rows $b_i^\top$, with $b_i\in\R^n$, are the \emph{spectral coordinates} of the observations. They play the role of Gale vectors for a general kernel. As for orthonormal Gale vectors, only their lengths $\|b_i\|^2=S_{ii}$ and inner products $b_i^\top b_j=S_{ij}$ matter, since any other factor of $S$ has the form $BO$ with $O$ orthogonal. The length $\|b_i\|^2$ generalises \eqref{eq:leverage}: a perturbation of size one at observation $i$ leaves a soft residual of size $S_{ii}$ at that observation.

Two quadratic quantities must be kept apart. The quadratic term of the marginal likelihood is $y^\top Qy=y^\top Sy/\sigma^2=\|B^\top y\|^2/\sigma^2$, and the group test of Section~\ref{sec:tests} is built from it. The soft residual energy $\|Sy\|^2=y^\top S^2y$ instead enters the noise estimate of Section~\ref{sec:noise}. The two coincide for an exact projector but not for $S$.

\subsection{The residual matrix in the small-noise limit}
\begin{lemma}\label{prop:limit}
For a fixed positive-semidefinite $K$,
\begin{equation}
 \lim_{\sigma^2\downarrow0}\sigma^2(K+\sigma^2I)^{-1}=P_N,
\end{equation}
where $P_N$ is the orthogonal projector onto $\nul(K)$. For any factorisation $K=\Phi\Sigma_w\Phi^\top$ with positive definite $\Sigma_w$, this is the residual projector \eqref{eq:projector}, since $\nul(K)=\nul(\Phi^\top)$ by \eqref{eq:null-kernel-feature}.
\end{lemma}
\begin{proof}
In \eqref{eq:spectrum}, a zero eigenvalue has weight one for every $\sigma^2>0$, while each positive eigenvalue has weight $\sigma^2/(\lambda_j+\sigma^2)\to0$. Since the matrix is finite dimensional, convergence holds in operator norm.
\end{proof}
Together with \eqref{eq:null-kernel-feature}, this shows that for a finite-rank kernel the residual matrix $S$ converges to the projector $P_N=GG^\top$ of the underlying feature configuration. This is a fixed-design limit. For full-rank $K$ the limit is zero, even when several eigenvalues are very small. Truncating the spectrum at a positive noise level is therefore an approximation, not an exact Gale transform of a full-rank kernel. Convergence of the matrix also does not make the GP test and the residual test identical, because their degrees of freedom can differ (Section~\ref{sec:power}).

\subsection{Hard structural relations and soft prior restrictions}\label{sec:hardsoft}
A finite-rank GP imposes hard structural relations, exactly as a linear model does. Four names used in this paper refer to this one object seen from different sides: the null space of $K$, the combinations of function values that the prior fixes at zero, the consistency relations of the model at the design, and the residual space $\nul(\Phi^\top)$, whose orthonormal basis $G$ has the Gale vectors as its rows. A full-rank GP imposes only soft prior restrictions. It can fit any data, but it pays a price in residual energy for patterns that its prior regards as implausible. Table~\ref{tab:dictionary} lists the corresponding objects. We write \emph{Gale vectors} and \emph{structural} only for the exact case, and \emph{spectral coordinates} and \emph{soft} for the general case.

\begin{table}[!htbp]
\centering
\small
\begin{tabular}{p{0.24\linewidth}p{0.33\linewidth}p{0.33\linewidth}}
\toprule
 & Finite-rank kernel (exact) & General kernel (spectral)\\
\midrule
Coordinates of observation $i$ & Gale vector $g_i$, row of $G$ & spectral coordinates $b_i$, row of $B=S^{1/2}$\\
Residual operator & projector $P_N=GG^\top$ & residual matrix $S=BB^\top=\sigma^2Q$\\
Restriction & hard structural relation, $v^\top f=0$ for $v\in\nul(K)$ & soft prior restriction, $\Var(u_j^\top f)=\lambda_j\ll\sigma^2$\\
A departure along it is & a structural contradiction, $\eta\propto1/\sigma_0^2$ & improbable under the prior, $\eta\to(u_j^\top\delta)^2/\lambda_j$ as $\sigma_0\to0$\\
Visible fraction at $i$ & $\|g_i\|^2=1-h_{ii}$ & $\|b_i\|^2=S_{ii}$\\
Minimal witnesses & circuits (Proposition~\ref{prop:circuits}) & no direct combinatorial analogue\\
Small-noise limit & $S\to P_N$ & $S\to0$\\
\bottomrule
\end{tabular}
\caption{Dictionary between the exact finite-rank construction and its spectral counterpart for general kernels. Here $\eta$ is the noncentrality of Section~\ref{sec:tests} and $\sigma_0^2$ the noise variance.}\label{tab:dictionary}
\end{table}

\section{Group consistency tests and their power}\label{sec:tests}
A departure can be detectable for two different reasons. It may violate an exact relation of the design, in which case its signal-to-noise ratio can grow without bound as the measurements become more precise. Or it may remain representable by the model and be rejected only because the prior assigns it little probability. The distinction is visible directly in the noncentrality of the test.

Before any test is chosen, the precision matrix $Q$ already describes how sensitive a design is to every departure from the model. Suppose that the responses contain a deterministic discrepancy $\delta\in\R^n$ in addition to the GP and the noise, $y=f+\delta+\varepsilon$ with $\varepsilon\sim\mathcal N(0,\sigma_0^2I)$. Section~\ref{sec:power} shows that the test of all observations then has noncentrality
\begin{equation}\label{eq:sensitivity}
 \eta=\delta^\top Q\delta
 =\sum_{j=1}^n\frac{(u_j^\top\delta)^2}{\lambda_j+\sigma_0^2}.
\end{equation}
Here $\eta$ is the noncentrality parameter of the chi-square test: larger values mean that the alternative is more strongly separated from the null and is therefore easier to detect. The matrix $Q$ and the eigenpairs $(\lambda_j,u_j)$ of $K$ are those of Section~\ref{sec:gp}, with $\sigma^2=\sigma_0^2$.

The spectrum of $K$ shows where this detectability comes from (Figure~\ref{fig:schematic}). Along an eigenvector in $\nul(K)$, the contribution to $\eta$ is $(u_j^\top\delta)^2/\sigma_0^2$. This part of the departure violates an exact consistency relation, so its detectability grows without bound as the noise decreases. We call this \emph{structural} detection. Along an eigenvector with $\lambda_j\gg\sigma_0^2$, the contribution is bounded by $(u_j^\top\delta)^2/\lambda_j$ as the noise decreases. Such a departure remains representable by the model and can only be judged improbable under the prior, so we call this \emph{prior-based} detection.

A soft prior restriction, with $0<\lambda_j\ll\sigma_0^2$, initially behaves like a structural direction because its contribution is approximately $(u_j^\top\delta)^2/\sigma_0^2$ while the noise dominates $\lambda_j$. As $\sigma_0^2\to0$, however, the contribution approaches $(u_j^\top\delta)^2/\lambda_j$ and therefore remains bounded. We reserve the word \emph{structural} for directions in $\nul(K)$.

At the end-point design of Figure~\ref{fig:teaser}(a), curvature lies entirely in the prior-based part. The GP can therefore reject a sufficiently large curvature even though no consistency relation involving curvature is violated (Figure~\ref{fig:schematic}(b)). We next express this distinction for tests of groups of observations and for families of departures.
\begin{figure}[!htbp]
\centering
\includegraphics[width=\linewidth]{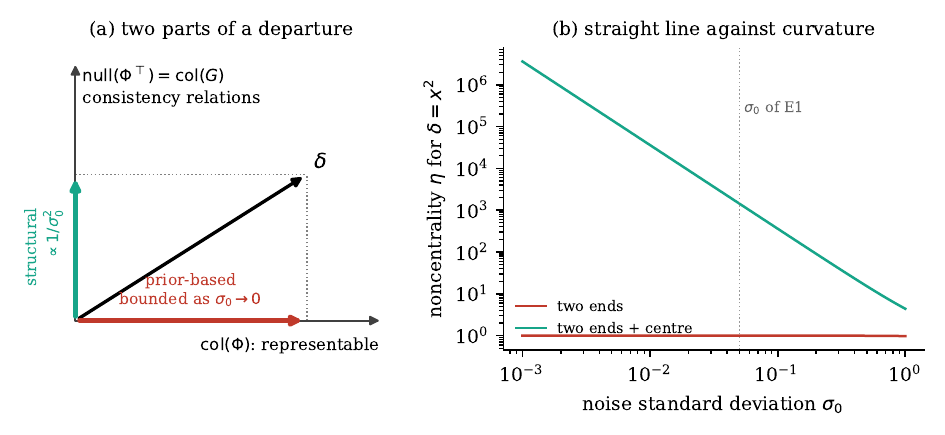}
\caption{Structural and prior-based detection for a finite-rank kernel. (a) A departure $\delta$ splits into a part that the model can represent, in $\col(\Phi)$, and a part that violates consistency relations, in $\nul(\Phi^\top)$. The first can be detected only through the prior, with a noncentrality that stays bounded as the noise decreases. The second is detected structurally, with a noncentrality proportional to $1/\sigma_0^2$. (b) Noncentrality \eqref{eq:sensitivity} of the test of all $40$ observations against a unit curvature $\delta(x)=x^2$, under the straight-line kernel $1+xx'$, as a function of the noise standard deviation $\sigma_0$, for the two designs of Figure~\ref{fig:teaser}(a). At the end-point design the curvature is an intercept shift, and its noncentrality stays close to $1$. With four inputs at the centre it grows as $1/\sigma_0^2$.}\label{fig:schematic}
\end{figure}

\subsection{The group test}
The group test asks a simple question. Can the responses of a group of observations be predicted from all the other responses, within the uncertainty of that prediction? If not, the group is incompatible with the model.

To make this precise, split the observation indices $\{1,\ldots,n\}$ into a set $C$ of \emph{candidate} observations, which we test, and the set $R$ of the remaining, \emph{retained} observations. A subscript selects entries: $y_C\in\R^{|C|}$ contains the responses $y_i$ with $i\in C$, and $A_{RC}\in\R^{|R|\times|C|}$ is the block of $A$ with rows in $R$ and columns in $C$, so that, after reordering,
\begin{equation}
 A=\begin{pmatrix}A_{RR}&A_{RC}\\A_{CR}&A_{CC}\end{pmatrix},\qquad A_{CR}=A_{RC}^\top.
\end{equation}
For a fixed noise variance $\sigma^2$, Gaussian conditioning (Appendix~\ref{app:block}) gives the prediction of $y_C$ from $y_R$ and its uncertainty,
\begin{equation}\label{eq:conditional}
 y_C\mid y_R\sim\mathcal N(A_{CR}A_{RR}^{-1}y_R,V_C),\qquad
 V_C=Q_{CC}^{-1}.
\end{equation}
The prediction error $z_C=y_C-A_{CR}A_{RR}^{-1}y_R\in\R^{|C|}$ is the cross-validation residual of the group \cite{dubrule1983,sundararajan2001,ginsbourger2025}. The group test statistic measures its size against its covariance,
\begin{equation}\label{eq:group}
 T_C=z_C^\top V_C^{-1}z_C.
\end{equation}
In words, $T_C$ measures how surprising the responses in $C$ are, given all other responses, after accounting for the uncertainty of the prediction and for the correlations within the group. Appendix~\ref{app:block} shows two equivalent forms. With the representer weights $\alpha=Qy$ of the posterior mean \cite{rasmussen2006}, $T_C=\alpha_C^\top Q_{CC}^{-1}\alpha_C=e_C^\top S_{CC}^{-1}e_C/\sigma^2$. The last form uses only the soft residual and the residual matrix, so a test computed in spectral coordinates is exactly the classical multiple-case test, and changing coordinates cannot change which groups are rejected.

Two different null hypotheses can be tested, and they must be kept apart. The \emph{structural null} of a linear model states only that the noiseless responses lie in $\col(\Phi)$, with fixed but unknown coefficients. Its residual test $\|G^\top y\|^2/\sigma^2\sim\chi^2_{n-r}$ uses the consistency relations alone. The GP null defined next also fixes a prior on the coefficients or the function, the mean and the noise. A rejection of the GP null can therefore come from the prior, the mean or the noise specification, and it does not by itself show that a structural assumption, such as the degree of a polynomial, is violated.

We call the model \emph{fully specified} when its mean function, its kernel including all hyperparameters, and its noise variance are fixed in advance rather than estimated from the data under test. The \emph{GP null} is the null hypothesis that the responses were generated by this model, that is, $y\sim\mathcal N(0,A)$ with the specified kernel and noise variance. Under the fully specified GP null,
\begin{equation}\label{eq:chisq}
 T_C\mid y_R\sim\chi^2_{|C|},
\end{equation}
exactly, for a fixed candidate group. Hyperparameters estimated from the same data generally destroy this exact calibration \cite{bachoc2013}, and Section~\ref{sec:noise} shows that a fitted noise variance can absorb the very discrepancy under test, so known measurement noise, replicates or a separate reference set are preferable.

\subsection{Power of the group test}\label{sec:power}
The \emph{power} of a test is the probability that it rejects the model when the model is wrong. For chi-square tests it is governed by a single number, the noncentrality. Proposition~\ref{prop:power} is a classical consequence of Gaussian conditioning and of the distribution of quadratic forms in normal variables, stated in the form needed here. Its proof is in Appendix~\ref{app:proofs}.
\begin{proposition}\label{prop:power}
Let $y=f+\delta+\varepsilon$, with $f\sim\mathcal{GP}(0,k)$, $\varepsilon\sim\mathcal N(0,\sigma_0^2I)$, and a deterministic discrepancy $\delta\in\R^n$, and compute $T_C$ in \eqref{eq:group} with the same kernel $k$ and the true noise variance $\sigma^2=\sigma_0^2$. Then, conditionally on $y_R$, and hence also unconditionally, $T_C$ has a noncentral chi-square distribution with $|C|$ degrees of freedom and noncentrality
\begin{equation}\label{eq:noncentralgeneral}
 \eta_C=\Delta_C^\top Q_{CC}\Delta_C=(Q\delta)_C^\top Q_{CC}^{-1}(Q\delta)_C,
\end{equation}
where $\Delta_C=\delta_C-A_{CR}A_{RR}^{-1}\delta_R\in\R^{|C|}$ is the part of the discrepancy on $C$ that the remaining observations do not predict. If $\delta$ is zero outside $C$, then $\Delta_C=\delta_C$ and $\eta_C=\delta_C^\top Q_{CC}\delta_C$.
\end{proposition}
In words, the test is more likely to reject the larger the part of the discrepancy that the other observations cannot predict, measured against the uncertainty of that prediction. For $C=\{1,\ldots,n\}$, $\Delta_C=\delta$ and \eqref{eq:noncentralgeneral} reduces to \eqref{eq:sensitivity}. The power at significance level $\beta$ is $\Pr(T_C>\chi^2_{|C|,1-\beta})$ under the noncentral chi-square distribution, where $\chi^2_{k,q}$ denotes the $q$-quantile of the chi-square distribution with $k$ degrees of freedom. It holds for every kernel and every noise level, and it depends only on the kernel, the noise variance, the design and the discrepancy, so it can be computed before any response is observed.

\subsection{Geometric meaning}
Proposition~\ref{prop:power} has a geometric form (Appendix~\ref{app:proofs}),
\begin{equation}\label{eq:noncentralsoft}
 \eta_C=\frac{\|\Pi_C B^\top\delta\|^2}{\sigma_0^2},
 \qquad\text{and}\qquad
 \eta_C=\frac{\|B_C^\top\delta_C\|^2}{\sigma_0^2}
 \quad\text{if $\delta$ is zero outside $C$},
\end{equation}
where $B_C\in\R^{|C|\times n}$ contains the spectral coordinates
$b_i$, $i\in C$, of \eqref{eq:spectralcoords} as its rows, and
$\Pi_C\in\R^{n\times n}$ is the orthogonal projector onto their span.
Each observation has a vector $b_i$, and the discrepancy is represented
by the weighted sum
\[
B^\top\delta=\sum_i \delta_i b_i\in\R^n.
\]
The noncentrality, and hence the power, is determined by how much of
this discrepancy vector lies in the span of the spectral coordinates of
the tested observations, relative to the noise level. Short spectral
coordinates make local departures less visible, while nearly aligned
coordinates provide similar rather than independent directions of
detection.

For finite-rank kernels the picture becomes exact. For a discrepancy supported on $C$ and a fixed finite-rank kernel, Lemma~\ref{prop:limit} gives $\sigma_0^2Q_{CC}\to(P_N)_{CC}=G_CG_C^\top$ as $\sigma_0^2\to0$, where $G_C\in\R^{|C|\times(n-r)}$ contains the Gale vectors of the observations in $C$ as its rows. Hence
\begin{equation}\label{eq:noncentral}
 \lim_{\sigma_0^2\to0}\sigma_0^2\eta_C=\|G_C^\top\delta_C\|^2,
\end{equation}
so that $\eta_C\approx\|G_C^\top\delta_C\|^2/\sigma_0^2$ for small noise whenever $G_C^\top\delta_C\neq0$. This is the noncentrality of the residual test $(P_Ny)_C^\top(P_N)_{CC}^+(P_Ny)_C/\sigma_0^2$ in the linear model \eqref{eq:linear}. The two tests still differ. The GP test has $|C|$ degrees of freedom and the residual test $\operatorname{rank}((P_N)_{CC})$, so their thresholds and their power differ even when their noncentralities agree. In both cases the power depends on the design only through the lengths of, and the angles between, the vectors of the affected observations. These are Gale vectors for a finite-rank kernel and spectral coordinates otherwise. Three consequences follow. We state them for exact Gale vectors and the residual tests of the linear model.
\begin{enumerate}
\item \emph{A departure that the model can represent violates no relation.} At the end-point design of Figure~\ref{fig:teaser}(a), the curvature $cx^2$ takes the same value $c$ at both inputs, so it equals an intercept shift, and no residual test can see it, however large $c$ is. The GP test can still reject, because a large intercept is improbable under the prior on the coefficients. This prior-based power grows with $c$ but not as the noise decreases, and it says nothing about curvature (experiment E1).
\item \emph{Offsets at observations with high leverage are hard to detect.} Such observations have short Gale vectors. In Example~\ref{ex:square}, an offset at a corner keeps $30\%$ of its squared size in the tests, and an offset at the centre $80\%$.
\item \emph{Observations with identical Gale vectors cannot be told apart.} In Example~\ref{ex:square}, offsets $c$ at corners $1$ and $3$ produce the same failures of the relations as an offset $2c$ at corner $1$.
\end{enumerate}

\subsection{Structured departures}\label{sec:structured}
A structural hypothesis is usually tested against a family of departures rather than a single one, such as all curvatures or all interactions. Let the departures of interest be $\delta=D\theta$ for a matrix $D\in\R^{n\times p}$, whose columns are $p$ basis directions such as a curvature or an interaction, and an unknown $\theta\in\R^p$.
\begin{proposition}\label{prop:family}
Under the assumptions of Proposition~\ref{prop:power}, the noncentrality of $T_C$ against $\delta=D\theta$ is the quadratic form $\eta_C(\theta)=\theta^\top M_C\theta$ with $M_C=D^\top Q_{:C}Q_{CC}^{-1}Q_{C:}D\in\R^{p\times p}$, where $Q_{:C}$ denotes the columns of $Q$ indexed by $C$. The least detectable unit departure has noncentrality $\min_{\|\theta\|=1}\eta_C(\theta)=\lambda_{\min}(M_C)$, and for $\theta\sim\mathcal N(0,\tau^2I_p)$ the mean noncentrality is $\mathbb E\,\eta_C=\tau^2\tr M_C$. For the test of all observations, $C=\{1,\ldots,n\}$, $M_C=D^\top QD$, and for a fixed finite-rank kernel
\begin{equation}\label{eq:family}
 \lim_{\sigma_0^2\to0}\sigma_0^2M_C=(G^\top D)^\top(G^\top D).
\end{equation}
The departures $\theta\in\nul(G^\top D)$ violate no consistency relation of the design and are invisible to every residual test. The family is \emph{structurally testable} in every direction exactly when $G^\top D$ has full column rank $p$, that is, when no non-zero departure in the family is representable by the model at the design.
\end{proposition}
In words, for a whole family of departures one small matrix says which members a design can detect structurally and which only through the prior. Both quantities depend on how the columns of $D$ are scaled, so they are comparable only under a stated amplitude convention, and we take the columns of $D$ to be linearly independent. The mean noncentrality is not the average power. Average power averages the noncentral chi-square tail probability over the departures, as in experiment E2, and need not rank designs in the same order. Both are nevertheless computable from the design before any response is observed. Structural testability is characterised by the rank of $G^\top D$. The GP test still gives every non-zero departure a positive noncentrality, since $Q$ is positive definite, but for $\theta\in\nul(G^\top D)$ this is the prior-based part of Section~\ref{sec:tests}: it does not grow as the noise decreases, as at the end-point design.

The columns of $G^\top D$ are the Gale images $\sum_iD_{ij}g_i$ of the basis directions of the family, so the Gale configuration shows which structured departures a design can detect structurally. It also defines a \emph{minimal structurally informative design}, the smallest design for which $G^\top D$ has full column rank. For a polynomial of degree $q$ against a departure of degree $q+1$, this is any design with $q+2$ distinct inputs (Example~\ref{ex:veronese}). Failure to reject may consequently reflect an untestable discrepancy, insufficient power, or genuine consistency. Propositions~\ref{prop:power} and~\ref{prop:family} do not decide between these from the data, but they identify in advance which discrepancies no consistency relation of the design can reveal and which it can detect only weakly. Because adding an observation changes the configuration, \eqref{eq:noncentralgeneral} also gives the power that the test would have after a candidate new observation, before that observation is made. Experiments E1 to E3 use this to choose where to observe.

\subsection{Relation to discrimination design and Fisher information}\label{sec:discrimination}
Choosing observations to maximise the power of a lack-of-fit test is a classical idea. Designs that detect the inadequacy of a regression model were constructed in \cite{atkinson1972,jones1978}, building on the treatment of bias from omitted terms in \cite{box1959}. The best-known criterion is T-optimality \cite{atkinson1975}. One of two rival regression models is assumed to be true, with known parameters, and the design maximises the sum of squared deviations between the true model and the best fit of the rival model. When the rival is the linear model \eqref{eq:linear} and the true model differs from it by $\delta$, this sum is $\min_w\|\delta-\Phi w\|^2=\|P_N\delta\|^2=\|G^\top\delta\|^2$, which is $\sigma_0^2$ times the noncentrality of the residual test, the limit in \eqref{eq:noncentral}. A T-optimal design therefore maximises the length of the Gale image of the departure. In plain terms, T-optimality asks whether the specified departure can still be mimicked by the null model on the chosen design. If the minimum is zero, the design cannot distinguish the two models structurally, and if it is large, the departure necessarily leaves a large residual. At the end-point design of Figure~\ref{fig:teaser}(a), $x^2$ cannot be told apart from an intercept shift, so the T-optimality criterion is zero, and adding centre points makes the curvature leave a non-zero residual. Like experiment E1, it requires the departure to be specified in advance. Properties of such designs are studied in \cite{dette2009}, extensions to non-normal models in \cite{lopezfidalgo2007}, and the general theory in \cite{atkinson2007}. Sequential and Bayesian approaches choose experiments to discriminate between mechanistic models \cite{box1967} or, with GP models and GP surrogates, between kernels and between mechanistic models \cite{gardner2015,olofsson2018}.

Two further classical quantities coincide with these noncentralities. For a fixed discrepancy, the conditional distributions of $y_C$ given $y_R$ under the alternative and under the null are Gaussian with the same covariance $V_C$ and means that differ by $\Delta_C$, so $\eta_C$ is twice their Kullback--Leibler divergence. For the family $\delta=D\theta$ and the test of all observations, the responses are distributed as $\mathcal N(D\theta,A)$, whose Fisher information about $\theta$ is $D^\top QD=M_C$. The worst-case criterion of Proposition~\ref{prop:family} is therefore E-optimality for the departure parameters \cite{atkinson2007}, and the structural part of this information is the part that grows as the noise decreases.

For finite-rank models and a specified alternative, maximising the structural noncentrality is therefore classical discrimination design. It is T-optimality, and for Gaussian models with a common covariance it is also maximisation of the \emph{projected} Kullback--Leibler divergence \cite{lopezfidalgo2007},
$\min_w\|\delta-\Phi w\|^2/2\sigma_0^2=\|G^\top\delta\|^2/2\sigma_0^2$.
The rival model is fitted by its best coefficients in this quantity. The \emph{full-GP} divergence $\frac12\delta^\top Q\delta$ is different because the GP prior also penalises coefficients within the representable subspace.

This difference is the point of the decomposition used here. It records whether the noncentrality comes from a component outside the model's representable space or from a component that remains representable but is improbable under the prior. For general kernels the same distinction is read spectrally, while in the finite-rank case the circuits identify the smallest sets of observations on which an exact contradiction can occur. At the end-point design of experiment E1, the T-optimality criterion and the projected divergence are zero for curvature because $x^2$ is representable there as an intercept shift. The GP test and the full-GP divergence can nevertheless be large because that intercept is improbable under the prior.

\section{Fitted noise as aggregate test failure}\label{sec:noise}
The calibration above assumes that the noise variance is known. When it is estimated from the same observations that are being tested, discrepancy energy can instead be absorbed into the fitted noise level. For fixed $K$, let $\hat\sigma^2$ maximise the log marginal likelihood, and write $\hat Q=Q(\hat\sigma^2)$, $\hat S=S(\hat\sigma^2)$ and $\hat e=\hat Sy$. At any positive interior maximum (Appendix~\ref{app:noise}),
\begin{equation}\label{eq:fixedpoint}
 \hat\sigma^2=\frac{\|\hat e\|^2}{\tr\hat S}
 =\frac{\|\hat e\|^2}{n-\tr(K\hat Q)}.
\end{equation}
In words, the fitted noise variance is the leftover misfit divided by the effective number of directions that the model cannot fit. The numerator is the total soft residual energy, that is, the aggregate failure of the hard relations and the soft prior restrictions. The denominator is an effective number of residual degrees of freedom, $n$ minus the effective number of parameters $\tr(K\hat Q)$ of a linear smoother \cite{hastie1990,wahba1990}. Equation~\eqref{eq:fixedpoint} is the GP form of MacKay's evidence fixed point \cite{mackay1992}. For a finite-rank kernel it approaches the classical estimator $\|P_Ny\|^2/(n-r)$ when $\hat\sigma^2$ is small compared with the non-zero eigenvalues of $K$. A discrepancy with energy along hard relations or soft prior restrictions passes almost unattenuated into the numerator and leaves the denominator nearly unchanged, so it can inflate $\hat\sigma^2$ (Appendix~\ref{app:noise}). The inflated value then enters the predictive variance at every input, and it also shrinks the statistic of the test of all observations, and typically the group statistics, so fitting the noise to the data under test can switch the checks off. In experiment S1 (Appendix~\ref{sec:detectionandforgetting}), three shifted observations among $40$ raise the fitted noise variance to almost four times its true value (Table~\ref{tab:s1}).

Tests computed with a fitted noise variance therefore tend to reject less often when the model is wrong, and a non-rejection is weaker evidence than it appears. The consequences for deleting data are examined in Appendix~\ref{sec:background}.

\section{Experiments}\label{sec:experiments}
Three experiments test the claims of Section~\ref{sec:tests} on controlled problems with known discrepancies. E1 asks whether a test can reject a model without testing the assumption in question, and whether the observation that best tests a GP is the one of largest predictive variance. E2 asks whether prospective power still helps when the location of the discrepancy is unknown. E3 asks which two-dimensional designs can test additivity. In all three the null model is fully specified, so Proposition~\ref{prop:power} applies exactly, and predicted powers are checked against rejection rates in $4000$ simulated data sets. All tests have level $0.05$. Full settings are given in Appendix~\ref{app:details}, and a simulation study of detection and forgetting is in Appendicex~\ref{sec:detectionandforgetting}.

\subsection{Structural contradiction or prior improbability (E1)}\label{sec:e1c}
\emph{Question.} Can a test reject a model without testing the assumption in question? \emph{Set-up.} The null model is the straight-line GP with kernel $1+xx'$ of Example~\ref{ex:polykernel}, whose prior $w\sim\mathcal N(0,I)$ on the intercept and slope is implied by the kernel, with noise variance $\sigma_0^2=0.0025$. The alternative adds a curvature $cx^2$. We compare the GP test of all $40$ observations, which tests the GP null, with the exact residual test $\|G^\top y\|^2/\sigma_0^2$, which tests the structural null and does not depend on the prior. The designs are $20$ inputs at each end of $[-1,1]$, $18$ at each end with $4$ at the centre, and $40$ equally spaced inputs.

\emph{Results.} At the end-point design, the $38$ consistency relations compare replicates only, and the residual test has power equal to its level at every curvature, up to $c=10$ (Figure~\ref{fig:structure}). The GP test does reject eventually, with power $0.74$ at $c=5$ and essentially $1$ at $c=10$, which is $200$ noise standard deviations. It rejects because a curvature at two inputs is an intercept shift that the prior finds improbable, a rejection that is valid for the GP null but says nothing about curvature. Over the range $c\leq0.25$ its power is at most $0.051$. With four inputs at the centre, the design contains circuits such as $y_i+y_j-2y_k=0$, with one observation at each of $-1$, $1$ and $0$, and both tests detect the curvature through them, with power $0.69$ and $0.68$ at $c=0.125$. Equally spaced inputs give $0.75$ and $0.73$. Adding centre points to a two-level design in order to test for curvature is standard practice in response-surface methodology \cite{boxdraper1987}, so the structural half of this result is classical. What the GP adds is the other half: the GP test does reject the end-point design for a large curvature, and the decomposition of Section~\ref{sec:tests} shows that this rejection comes from the prior, not from a violated relation. \emph{Conclusion.} A rejection by the GP test does not show that the suspected assumption is wrong, and the design that is optimal for estimating a straight line contains no relation that tests whether the line is straight.

\begin{figure}[!htbp]
\centering
\includegraphics[width=0.65\linewidth]{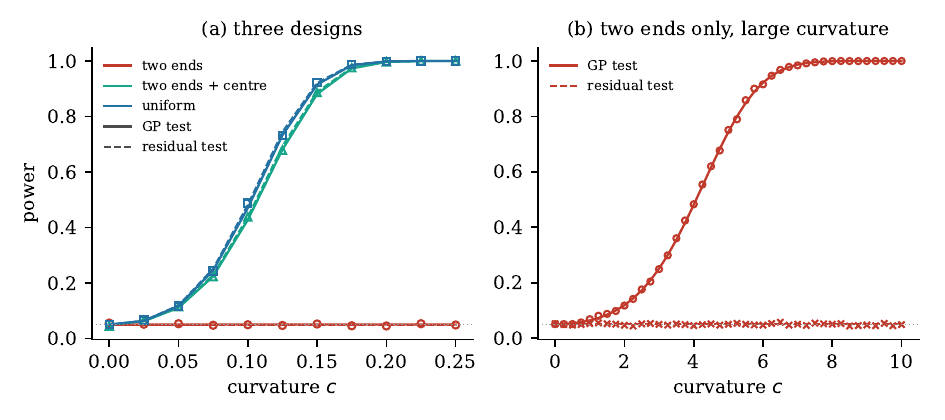}
\caption{E1, testing a straight-line assumption against a curvature $cx^2$. Solid lines show the power of the GP test predicted by Proposition~\ref{prop:power}, dashed lines the power of the exact residual test, and markers the rejection rates in $4000$ simulated data sets. The dotted line marks the level $0.05$. (a) Three designs of $40$ inputs. (b) The end-point design for large curvatures. The residual test stays at the level for every curvature, because no consistency relation of the end-point design involves curvature. The GP test eventually detects a very large curvature, as an intercept shift that is improbable under the prior.}\label{fig:structure}
\end{figure}

\subsection{Choosing for power, for discrimination or for predictive variance (E1, continued)}\label{sec:e1}
\emph{Question.} Is the observation that best tests a GP the one of largest predictive variance, and how does a classical discrimination criterion compare? \emph{Set-up.} The null model is a GP with an RBF kernel of lengthscale $0.3$ and unit output scale, zero mean and noise variance $\sigma_0^2=0.0025$, on designs of $40$ inputs in $[-1,1]$, of which $36$ are drawn uniformly and $4$ lie in a cluster at $0.2$. The alternative is a narrow bump $\delta(x)=0.2\exp(-(x-x_0)^2/2w^2)$ of width $w=0.04$ at a specified location $x_0$, tested with the group $C$ of observations within $0.1$ of $x_0$. Each of $401$ candidate inputs is scored in three ways:
\begin{itemize}
\item the \emph{power rule}: the power the test would have after adding the candidate, from \eqref{eq:noncentralgeneral};
\item the \emph{KL rule}: the full-GP Kullback--Leibler divergence $\frac12\delta^\top Q\delta$ between alternative and null after adding it;
\item the \emph{variance rule}: its current predictive variance.
\end{itemize}
We compare the choices over $20$ designs and $19$ bump locations, first with the true hyperparameters, and then with the lengthscale or the noise standard deviation wrong by a factor of $2$ when choosing.

\emph{Results.} Figure~\ref{fig:power}(b) first verifies the power formula against common offsets at four kinds of group. The simulated rejection rates agree with the prediction to within $0.022$ in every case, and an offset of four noise standard deviations is detected with probability $0.65$ at the boundary observation and $0.95$ at an interior one, because a GP absorbs an offset at the boundary more easily ($S_{ii}=0.34$ against $0.79$). For the bump at $x_0=0.1$ (Figure~\ref{fig:teaser}(b)), the current design has power $0.11$. Adding the observation of largest predictive variance, at the boundary, leaves the power at $0.11$. Adding the observation of largest predicted power, at $x=0.1$, raises it to $0.64$, although the predictive variance there is fourteen times smaller. Table~\ref{tab:e1baselines} shows that this is typical. Averaged over designs and locations, the power rule raises the power from $0.48$ to $0.72$, whereas the variance rule gives $0.51$ and a random candidate $0.49$. The KL rule gives $0.72$ as well and chooses the same input as the power rule in $77\%$ of the cases. The two are close because both reward noncentrality against the specified bump. They differ because the power rule scores the conditional test of the group near $x_0$, whose membership and degrees of freedom change with the candidate, whereas the KL rule scores all observations of the augmented design. When the hyperparameters assumed for choosing are wrong, the power rule still gives $0.67$ to $0.72$ under the true model, and its paired advantage over the variance rule stays between $0.16$ and $0.22$. Here only the choice uses the wrong hyperparameters. A test that used them as well would no longer be exactly calibrated (Section~\ref{sec:discussion}). \emph{Conclusion.} An observation that reduces predictive uncertainty is, in general, not the observation that tests the model. The prospective-power choice agrees with classical discrimination where it should, and it degrades gracefully when the lengthscale or noise level are misspecified by a factor of two.

\begin{figure}[!htbp]
\centering
\includegraphics[width=\linewidth]{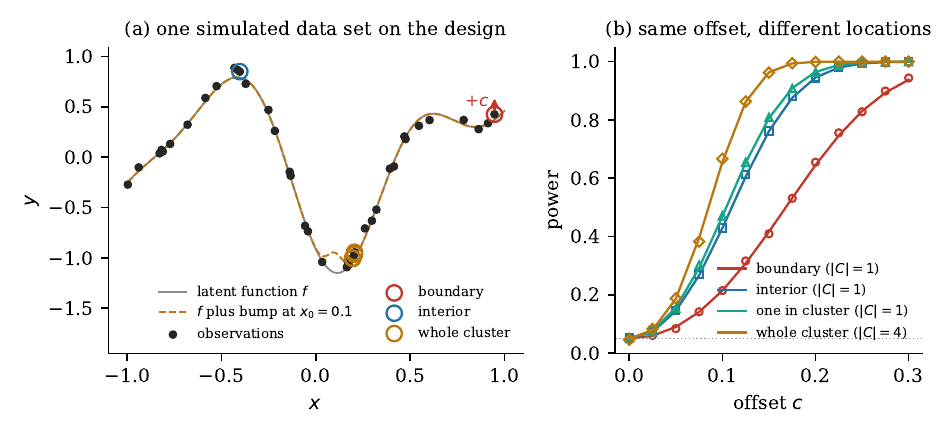}
\caption{E1, design-dependent power. (a) One data set simulated from the GP null on the design used in (b) and in Figure~\ref{fig:teaser}(b), with the latent function (solid), the latent function plus the bump at $x_0=0.1$ (dashed), and the observations without discrepancy (dots). Circles mark the boundary observation, an interior observation and the cluster. The arrow shows the offset $c=0.2$, four noise standard deviations, added at the boundary observation. (b) Verification: power of the group test against a common offset $c$ at four kinds of group in this design. Lines show the power predicted by Proposition~\ref{prop:power}, markers the rejection rate in $4000$ simulated data sets. The dotted line marks the level $0.05$. The candidate scores for the bump at $x_0=0.1$ are shown in Figure~\ref{fig:teaser}(b).}\label{fig:power}
\end{figure}

\begin{table}[!htbp]
\centering
\small
\begin{tabular}{lrrrr}
\toprule
Assumed when choosing & Power rule $\uparrow$ & KL rule $\uparrow$ & Variance rule $\uparrow$ & Power $-$ variance $\uparrow$\\
\midrule
True hyperparameters & 0.72 & 0.72 & 0.51 & 0.217 $\pm$ 0.008\\
Lengthscale $\times0.5$ & 0.67 & 0.67 & 0.51 & 0.156 $\pm$ 0.008\\
Lengthscale $\times2$ & 0.71 & 0.70 & 0.50 & 0.208 $\pm$ 0.008\\
Noise s.d.\ $\times0.5$ & 0.72 & 0.71 & 0.50 & 0.217 $\pm$ 0.008\\
Noise s.d.\ $\times2$ & 0.72 & 0.72 & 0.50 & 0.218 $\pm$ 0.008\\
\bottomrule
\end{tabular}
\caption{E1, power of the group test against a bump of specified location after one added observation, averaged over $20$ designs and $19$ bump locations. The rules choose under the hyperparameters in the first column, and every choice is evaluated under the true model (lengthscale $0.3$, noise variance $0.0025$). The current design has mean power $0.48$. The KL rule maximises the Kullback--Leibler divergence between the alternative and the null for all observations, and chooses the same input as the power rule in $77\%$ of the cases with the true hyperparameters. The last column is the paired difference between the power and variance rules, averaged per design, with one standard error over the $20$ designs.}\label{tab:e1baselines}
\end{table}

\subsection{Discrepancies of unknown location (E2)}\label{sec:e2}
\emph{Question.} Does prospective power still help when the location of the discrepancy is unknown? \emph{Set-up.} The null model and the $20$ starting designs are those of E1. The alternative is a family of bumps of height $0.2$ and width $0.04$ whose centre lies in $[0.3,0.9]$, with prior weights from a normal density centred at $0.6$, and the test is the group test of all observations in that region. Five observations are added one at a time from $61$ candidates in the region, either by maximising the prior-weighted mean of the rejection probabilities over the family (\emph{average power}) or their minimum (\emph{worst case}), or by one of five baselines: the prior-weighted mean Kullback--Leibler divergence (\emph{KL}), largest predictive variance, largest prior-weighted squared bump height (\emph{amplitude}), equally spaced inputs (\emph{space-filling}) and random inputs. The additions do not depend on any response, since the power depends only on the design. Averaging the power keeps the criterion tied to the test. A Bayesian design would instead maximise an expected information gain \cite{chaloner1995}.

\emph{Results.} After five observations the average-power rule reaches a mean power of $0.52$, the best of the six rules in each of the $20$ designs (Figure~\ref{fig:unknown}(a)). The strongest baseline is space-filling, at $0.46$, and the average-power rule exceeds it by $0.059\pm0.005$, paired over designs (one standard error). The KL rule reaches $0.520$, against $0.523$, so the two criteria choose nearly the same designs.

Two limits remain. The power against the least detectable bump stays at most $0.16$ for every rule (Figure~\ref{fig:unknown}(b)), and Proposition~\ref{prop:power} shows this before any observation is made. When the true bump differs from the assumed one (Table~\ref{tab:e2sens}), the average-power rule stays the best, but its gain over space-filling falls from $0.03$--$0.06$ for narrower or higher bumps to $0.01$--$0.02$ for wider or lower ones. \emph{Conclusion.} Prospective power remains useful within a specified family of alternatives, it agrees with a classical discrimination criterion, and its advantage shrinks as the alternative is misspecified.

\begin{figure}[!htbp]
\centering
\includegraphics[width=\linewidth]{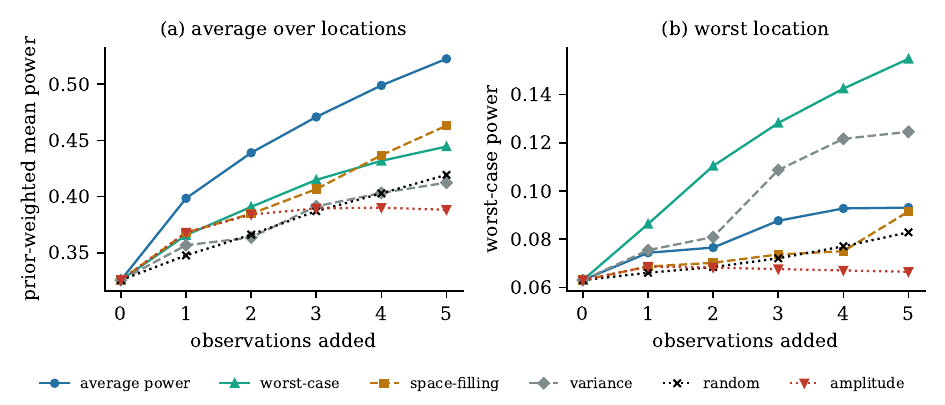}
\caption{E2, choosing five observations against a bump of unknown location in $[0.3,0.9]$, averaged over $20$ designs. (a) Prior-weighted mean power over the family of bump locations. (b) Power against the least detectable bump. The average-power and worst-case rules use Proposition~\ref{prop:power}. The other rules are baselines described in the text. The KL rule is not shown, since its curve is indistinguishable from that of the average-power rule.}\label{fig:unknown}
\end{figure}

\begin{table}[!htbp]
\centering
\small
\resizebox{\linewidth}{!}{\begin{tabular}{lrrrr}
\toprule
True discrepancy & Average power $\uparrow$ & Gain over space-filling $\uparrow$ & Gain over variance $\uparrow$ & Gain over random $\uparrow$\\
\midrule
as assumed & 0.52 & 0.059 $\pm$ 0.005 & 0.110 $\pm$ 0.009 & 0.102 $\pm$ 0.007\\
narrower bump (width 0.02) & 0.44 & 0.057 $\pm$ 0.004 & 0.111 $\pm$ 0.006 & 0.081 $\pm$ 0.004\\
wider bump (width 0.08) & 0.24 & 0.017 $\pm$ 0.002 & 0.025 $\pm$ 0.003 & 0.029 $\pm$ 0.003\\
lower bump (height 0.1) & 0.14 & 0.011 $\pm$ 0.001 & 0.022 $\pm$ 0.002 & 0.019 $\pm$ 0.001\\
higher bump (height 0.3) & 0.87 & 0.062 $\pm$ 0.010 & 0.112 $\pm$ 0.017 & 0.125 $\pm$ 0.017\\
uniform location prior & 0.46 & 0.034 $\pm$ 0.004 & 0.060 $\pm$ 0.007 & 0.057 $\pm$ 0.005\\
\bottomrule
\end{tabular}}
\caption{E2, sensitivity to a misspecified alternative. The designs were chosen for bumps of height $0.2$ and width $0.04$ with a normal location prior centred at $0.6$, and are evaluated here under the true discrepancy in the first column, without re-selection. The second column is the prior-weighted mean power of the average-power rule after five additions. The other columns are paired differences in mean power between the average-power rule and each baseline, with one standard error over the $20$ starting designs.}\label{tab:e2sens}
\end{table}

\subsection{Testing additivity in two dimensions (E3)}\label{sec:e3}
\emph{Question.} Which two-dimensional designs can test additivity? \emph{Set-up.} The null model is an additive GP on $[-1,1]^2$, with kernel $k(x,x')=k_1(x_1,x_1')+k_2(x_2,x_2')$, where $k_1$ and $k_2$ are RBF kernels with lengthscale $0.5$ and unit output scale, and noise variance $\sigma_0^2=0.0025$. Although both components have full rank, their sum imposes exact relations (Section~\ref{sec:gp}), so this is a GP with genuine exact structure that is not a polynomial regression. The alternative is an interaction $c\,x_1x_2$. We compare designs of $16$ inputs: a $4\times4$ grid, Latin hypercube samples and uniformly random inputs ($50$ draws each), and a design chosen greedily from a $21\times21$ candidate grid by the noncentrality of the GP test against $x_1x_2$.

\emph{Results.} The grid carries the $(4-1)(4-1)=9$ mixed-difference relations \eqref{eq:mixed}. They are spanned by the mixed differences of its $2\times2$ subgrids, and each of these is a circuit of four observations: any assignment of values on three of the four inputs can be matched by an additive function, whereas the values on all four must satisfy the mixed-difference relation. The GP test of the grid detects $c=0.1$ with probability $0.81$, all of it structurally (Figure~\ref{fig:additivity}). A Latin hypercube has no two inputs that share a coordinate, so for every set $C$ of its inputs $K_{CC}$ is non-singular and no set supports a circuit. The Latin hypercube and random designs detect the same interaction with probability $0.09$ and $0.08$ on average, through the prior and soft prior restrictions only. The power-optimised design places its inputs on the two lines $x_2=\pm1$, which forms an $8\times2$ grid with $7$ relations, and it detects $c=0.1$ with probability $0.999$. On the grid and on the power-optimised design, simulated rejection rates agree with the prediction to within $0.002$. \emph{Conclusion.} A design whose inputs share coordinates contains circuits that test additivity, which a Latin hypercube lacks, and a design chosen for power finds such circuits by itself, at the price of being tailored to the interaction it was chosen for.

\begin{figure}[!htbp]
\centering
\includegraphics[width=\linewidth]{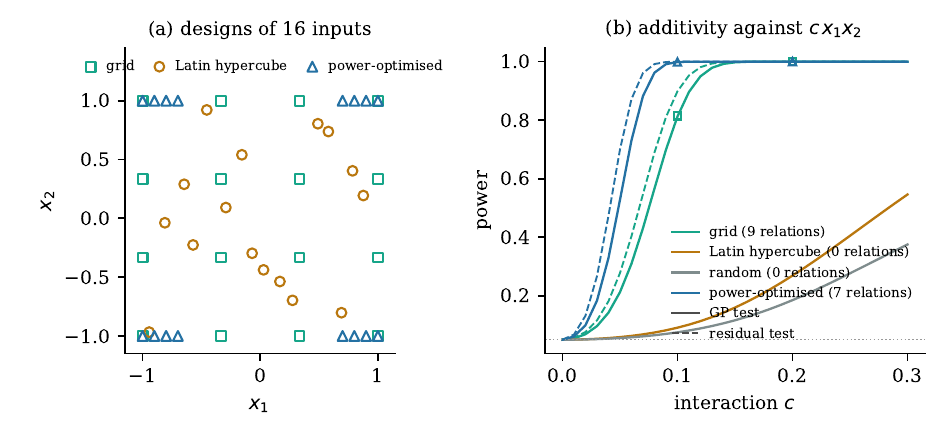}
\caption{E3, testing an additive GP against an interaction $c\,x_1x_2$ with $16$ inputs. (a) The grid, one Latin hypercube draw and the power-optimised design. (b) Power of the GP test (solid) and of the exact residual test on the null space of $K$ (dashed), with the number of exact consistency relations of each design. Latin hypercube and random curves are means over $50$ draws and have no residual test. Markers show simulated rejection rates of the GP test.}\label{fig:additivity}
\end{figure}

\section{Discussion}\label{sec:discussion}
\subsection{What the experiments show}
Experiment E1 exposes a distinction that is easy to miss when power is treated as a single number. At the end-point design, the GP can reject a sufficiently large curvature even though the observations contain no structural information about curvature. The curvature is representable there as an intercept shift, and the rejection comes from that shift being improbable under the prior. Moving only four observations to the centre changes the nature of the evidence: curvature then violates an exact consistency relation.

The same issue appears in two dimensions. The $4\times4$ grid in E3 contains mixed-difference circuits that test additivity directly, while a Latin hypercube of the same size contains none. Both designs may still assign low probability to an interaction under the GP prior, but only the grid contains an exact algebraic check of additivity. This is why a scalar discrimination score can be insufficient for interpretation: the full-GP KL criterion can be large even when the projected KL and T-optimality criteria are zero.

The prospective power calculation itself behaved as expected. Proposition~\ref{prop:power} tracked the simulated rejection probabilities closely. In E1, choosing the next observation by predicted power increased average power from $0.48$ to $0.72$, compared with $0.51$ for predictive variance, and most of the advantage remained when the lengthscale or noise level used for choosing was wrong by a factor of two. Where the problem reduces to classical discrimination, as in E1 and E2, the resulting choices agree closely with those criteria.

\subsection{Limits of the present analysis}
The exact calibration assumes a fixed mean, kernel and noise variance with Gaussian homoscedastic noise. E1 suggests that the \emph{choice} of design is not overly fragile to moderate hyperparameter misspecification, but the nominal test level is no longer exact when those hyperparameters are estimated from the observations being tested. Section~\ref{sec:noise} gives one concrete failure mode: discrepancy energy can inflate the fitted noise variance and weaken the test.

The power calculation also requires a specified discrepancy or family of discrepancies. This is useful when a scientifically meaningful failure mode is known, but it can lead to designs that are deliberately specialised. The optimised design in E3 is a clear example. E2 shows a milder gain when the discrepancy location is unknown.

Two further restrictions are structural. The present results concern a fixed design tested once; repeated testing as observations arrive requires control of false rejections over time. They also concern linear consistency relations. Rank constraints and other algebraic hypotheses lead to determinantal or more general non-linear relations and need a different treatment. Finally, all experiments here are synthetic and low-dimensional. A rejection should therefore be read only as incompatibility with the specified model, not as identification of a faulty observation or assumption. Model repair---through a robust likelihood \cite{jylanki2011,kersting2007,ament2024}, an explicit discrepancy term \cite{kennedy2001}, or a physically informed mean \cite{deboi2022,ramosgarces2024}---is a separate decision.

\subsection{Implications for sequential design}
A model can fit the data because it is adequate or because the design never challenged it \cite{popper1959,box1980}. Testability is also different from usefulness for the optimisation task. In Bayesian optimisation, an evaluation used to challenge the model is an evaluation not used directly to improve the objective \cite{srinivas2010,bogunovic2021}. The distinction can be summarised as
\begin{equation}
 \boxed{\text{fit}\ \neq\ \text{testability}\ \neq\ \text{decision value}}.
\end{equation}

The results here make the middle term computable for a fixed design. Extending this to a sequence of decisions raises several separate problems. Repeated tests need time-uniform error control, for example through confidence sequences \cite{howard2021}. Non-linear structural hypotheses, such as low-rank interactions, replace null-space relations by algebraic varieties; for a Segre-type rank model, distances to the relevant secant variety and its Euclidean distance degree \cite{michalek2021,draisma2016} are natural analogues of the projector used here. An optimisation algorithm must also decide whether the next evaluation is worth spending on testing rather than on improvement. Finally, uncertainty in the discrepancy family and in the GP hyperparameters must be propagated into that decision rather than treated as fixed. These questions are deliberately left open here.

\section{Conclusion}
A GP can reject a departure even when the design never observed the structure that was supposedly violated. The end-point example makes this concrete: curvature can be rejected through an improbable intercept although no consistency relation of that design involves curvature.

For finite-rank kernels, the null space separates these cases exactly. Gale duality then shows which observations participate in the structural relations and which minimal subsets can witness a contradiction. For general kernels, the same question becomes spectral rather than algebraic. In both cases the calculation is available before the responses are measured, so testability can be inspected as a property of the design rather than inferred only after a failed model check.

What remains is the sequential problem. During optimisation the design, the GP and the structural hypothesis all evolve. Deciding when a new evaluation should improve the objective and when it should challenge the model requires error control, uncertainty over the assumed structure and a decision rule that trades testing against optimisation.

\section*{Funding}
Ivan De Boi is funded by the Research Foundation -- Flanders (FWO)
through postdoctoral fellowship 1217125N.

\section*{Data availability}
The data supporting the findings of this study are available upon
reasonable request from the corresponding author.

\appendix
\section{Block deletion identities}\label{app:block}
Order the indices as $(R,C)$ and write $A_{RR}$, $A_{RC}$ and $A_{CC}$ for the blocks of $A$. The Schur complement of $A_{RR}$ is $V_C=A_{CC}-A_{RC}^\top A_{RR}^{-1}A_{RC}\in\R^{|C|\times|C|}$, and
\begin{equation}
 A^{-1}=\begin{pmatrix}
 A_{RR}^{-1}+A_{RR}^{-1}A_{RC}V_C^{-1}A_{RC}^\top A_{RR}^{-1}&-A_{RR}^{-1}A_{RC}V_C^{-1}\\
 -V_C^{-1}A_{RC}^\top A_{RR}^{-1}&V_C^{-1}
 \end{pmatrix}.
\end{equation}
Hence $Q_{CC}=V_C^{-1}$, and reading off the $C$ block of $\alpha=A^{-1}y$ gives $\alpha_C=V_C^{-1}(y_C-A_{RC}^\top A_{RR}^{-1}y_R)=Q_{CC}z_C$, which establishes \eqref{eq:conditional} and $T_C=\alpha_C^\top Q_{CC}^{-1}\alpha_C$. Since $e=Sy=\sigma^2\alpha$ and $S_{CC}=\sigma^2Q_{CC}$, also $T_C=e_C^\top S_{CC}^{-1}e_C/\sigma^2$. Expanding the quadratic form,
\begin{equation}
 y^\top A^{-1}y=y_R^\top A_{RR}^{-1}y_R+z_C^\top V_C^{-1}z_C=y_R^\top A_{RR}^{-1}y_R+T_C,
\end{equation}
so $T_C$ is the drop in the quadratic term of the log marginal likelihood when $C$ is deleted, and $\det A=\det A_{RR}\det V_C$ gives the corresponding change in the log determinant.

\section{Proofs}\label{app:proofs}
\begin{proof}[Proof of Proposition~\ref{prop:power}]
The discrepancy shifts the conditional mean of $y_C$ given $y_R$ by $\delta_C$ and its prediction $A_{CR}A_{RR}^{-1}y_R$ by $A_{CR}A_{RR}^{-1}\delta_R$, so the mean of $z_C$ in \eqref{eq:conditional} becomes $\Delta_C$, while its covariance $V_C=Q_{CC}^{-1}$ is unchanged. Hence $V_C^{-1/2}z_C\mid y_R\sim\mathcal N(V_C^{-1/2}\Delta_C,I_{|C|})$, and $T_C=\|V_C^{-1/2}z_C\|^2$ is noncentral chi-square with noncentrality $\Delta_C^\top V_C^{-1}\Delta_C=\Delta_C^\top Q_{CC}\Delta_C$. The block inverse in Appendix~\ref{app:block} gives $(Q\delta)_C=Q_{CC}\Delta_C$, which yields the second form. For the geometric form \eqref{eq:noncentralsoft}, $Q=BB^\top/\sigma_0^2$ gives $(Q\delta)_C=B_CB^\top\delta/\sigma_0^2$ and $Q_{CC}=B_CB_C^\top/\sigma_0^2$, so the second form equals $\delta^\top BB_C^\top(B_CB_C^\top)^{-1}B_CB^\top\delta/\sigma_0^2=\|\Pi_CB^\top\delta\|^2/\sigma_0^2$. If $\delta_R=0$, then $\Delta_C=\delta_C$, and $\delta_C^\top Q_{CC}\delta_C=\|B_C^\top\delta_C\|^2/\sigma_0^2$.
\end{proof}
\begin{proof}[Proof of Proposition~\ref{prop:family}]
By linearity, $(Q\delta)_C=Q_{C:}D\theta$, and \eqref{eq:noncentralgeneral} gives $\eta_C(\theta)=\theta^\top M_C\theta$. The minimum over unit $\theta$ is the smallest eigenvalue, and the mean of a quadratic form in $\mathcal N(0,\tau^2I_p)$ is $\tau^2\tr M_C$. For $C=\{1,\ldots,n\}$, $Q_{:C}Q_{CC}^{-1}Q_{C:}=Q$, and $\sigma_0^2Q\to P_N=GG^\top$ as $\sigma_0^2\to0$ by Lemma~\ref{prop:limit}. Finally, $G^\top D\theta=0$ exactly when $D\theta\in\col(\Phi)$.
\end{proof}

\section{The noise fixed point}\label{app:noise}
For fixed $K$, the log marginal likelihood as a function of the noise variance is
\begin{equation}\label{eq:evidence}
 L(\sigma^2)=-\frac12y^\top Q(\sigma^2)y-\frac12\log\det A(\sigma^2)-\frac n2\log(2\pi),
\end{equation}
with derivative $L'(\sigma^2)=\frac12y^\top Q^2y-\frac12\tr Q$. Setting $L'(\hat\sigma^2)=0$, multiplying by $2\hat\sigma^4$ and using $\hat\sigma^2\hat Qy=\hat e$ and $\hat\sigma^2\tr\hat Q=\tr\hat S$ gives \eqref{eq:fixedpoint}. The second form follows from $\hat S=I-K\hat Q$. Both sides depend on $\hat\sigma^2$, so \eqref{eq:fixedpoint} is a stationarity identity rather than a closed-form estimator. It holds for the maximum of the marginal likelihood. The experiments of Appendix~\ref{sec:detectionandforgetting} fit the noise by maximum a posteriori estimation with a weak prior, for which it holds only approximately.

Now let $y=f+\delta+\varepsilon$, where $f\in\R^n$ is represented well by the model, $\delta\in\R^n$ is a deterministic discrepancy and $\varepsilon\sim\mathcal N(0,\sigma_0^2I)$. At any fixed value of $\sigma^2$,
\begin{equation}\label{eq:expected}
 \mathbb E\|Sy\|^2=\|S(f+\delta)\|^2+\sigma_0^2\tr(S^2).
\end{equation}
A discrepancy with energy along hard relations or soft prior restrictions passes almost unattenuated into the residual. At a fixed noise level it increases the expected numerator of \eqref{eq:fixedpoint} while leaving the denominator essentially unchanged, so it can raise $\hat\sigma^2$ above $\sigma_0^2$. This is not a general monotone result. The effect depends on the direction of the discrepancy, on its cross term with $f$, and on how $S$ and the denominator change with the fitted variance.

\section{Experimental details}\label{app:details}
\paragraph{E1.} The null model is a zero-mean GP with kernel $k(x,x')=\exp(-(x-x')^2/2\ell^2)$, $\ell=0.3$, and noise variance $\sigma_0^2=0.0025$. A design consists of $n=40$ inputs in $[-1,1]$, of which $36$ are drawn uniformly and $4$ lie in a tight cluster at $0.2\pm0.01$. The check of Figure~\ref{fig:power}(b) adds a common offset $c\in[0,0.3]$, up to six noise standard deviations, to four kinds of group in one design: the observation at the right boundary, an interior observation near $x=-0.4$, one member of the cluster, and the whole cluster. The predicted power is compared with the rejection rate in $4000$ data sets $y\sim\mathcal N(\delta,A)$ for each offset. For the bump, which is small but not zero outside the tested group, the power is computed from the general form \eqref{eq:noncentralgeneral}. The candidates are $401$ equally spaced inputs in $[-1,1]$, and one observation is added. A candidate joins the tested group when it lies within $0.1$ of $x_0$, so the degrees of freedom and the rejection threshold are recomputed for each candidate, and candidates are scored by power rather than by noncentrality. The predictive variance is that of the latent function, and the random choice is the mean power over all candidates. The detailed example uses $x_0=0.1$, and the sweep combines $20$ designs with $x_0\in\{-0.9,-0.8,\ldots,0.9\}$. For the mismatch check, the rules choose with lengthscale $0.15$ or $0.6$, or with noise variance $6.25\times10^{-4}$ or $0.01$, and each chosen design is evaluated, and tested, with the true hyperparameters. The straight-line part uses the kernel $1+xx'$, curvatures $c\leq0.25$ for all three designs and $c\leq10$ for the end-point design. The residual test has $n-r=38$ degrees of freedom at all three designs.

\paragraph{E2.} The bump centres lie on a grid of $31$ points in $[0.3,0.9]$, with prior weights proportional to a normal density with mean $0.6$ and standard deviation $0.15$. The tested group contains all observations in the region, including those that will be added. The random baseline is averaged over $10$ draws. The KL rule maximises the prior-weighted mean of $\frac12\delta^\top Q\delta$ over the family, for all observations of the augmented design. For the final design of the average-power rule in the first case, simulated rejection rates at five bump locations agree with the predicted power to within $0.01$. The sensitivity scenarios of Table~\ref{tab:e2sens} keep the designs and change the width, height or location prior of the true bump.

\paragraph{E3.} The Latin hypercube samples place one input in each of $16$ equal strips of each coordinate. The greedy design starts from the empty design and adds, at each step, the candidate that maximises the noncentrality of the test of all observations against $x_1x_2$, without repeating a candidate. An eigenvalue of $K$ is treated as zero when it is below $10^{-12}$ times the largest. The smallest eigenvalue of $K$ over all Latin hypercube and random draws is about $2\times10^{-11}$ times the largest, far below $\sigma_0^2$, so these designs do have soft prior restrictions, but the interaction has little energy along them.

\paragraph{Computation.} Every power evaluation solves a linear system of size $n$, and choosing an observation repeats this for every candidate. This is cheap for the designs used here, with at most $45$ observations, but not for large designs without low-rank or incremental updates.

\section{Background material}\label{sec:background}
This appendix collects background material that the main text shortened.

\paragraph{The F2I objective.} Forgetting to Improve (F2I) \cite{wendl2026} selects a retained set $R$ of observations that minimises the predictive variance over a finite target set $\mathcal A\subset\mathcal X$,
\begin{equation}\label{eq:f2i}
 J(R)=\sum_{x\in\mathcal A}\left[k(x,x)-k_{Rx}^\top(K_{RR}+\hat\sigma_R^2I_{|R|})^{-1}k_{Rx}+\hat\sigma_R^2\right],
\end{equation}
where $k_{Rx}$ has entries $k(x_i,x)$ for $i\in R$, $K_{RR}$ is the kernel matrix of the retained inputs, and $\hat\sigma_R^2$ is the noise variance refitted on them. By the noise fixed point (\ref{eq:fixedpoint}), every observation with a large soft residual raises the fitted noise, including honest observations whose residuals are large by chance. An objective that rewards a small $\hat\sigma_R^2$ therefore also rewards removing the upper tail of the honest noise distribution, which leads to overconfidence. A smaller predicted variance is not evidence that the model has become more correct.

\paragraph{Deletion at fixed noise.} Deleting observations interacts with this mechanism. For a prediction input $x\in\mathcal X$, let $k_{Cx}\in\R^{|C|}$ have entries $k(x_i,x)$ for $i\in C$, let $d_C(x)=k_{Cx}-A_{CR}A_{RR}^{-1}k_{Rx}\in\R^{|C|}$ be the conditional covariance between $f(x)$ and $y_C$ given $y_R$, with $V_C$ as in (\ref{eq:conditional}), and let $v_R(x)$ denote the posterior variance of $f(x)$ given $y_R$, all at the same noise variance $\sigma^2$.
\begin{proposition}\label{prop:variance}
The latent posterior variances satisfy
\begin{equation}\label{eq:variancecost}
 v_R(x)-v_{R\cup C}(x)=d_C(x)^\top V_C^{-1}d_C(x)\geq0.
\end{equation}
\end{proposition}
\begin{proof}Conditioning first on $y_R$ and then on $y_C$, Gaussian conditioning on $y_C$ subtracts the displayed quadratic form from the variance. The matrix $V_C$ is positive definite because $\sigma^2>0$.\end{proof}
At fixed kernel and noise, deleting data therefore never reduces uncertainty. Any improvement after deletion must come from re-estimating model quantities, here the noise variance, rather than from the deletion itself. Deleting observations that carry discrepancy energy removes more from the numerator of (\ref{eq:fixedpoint}) than from its denominator, and so tends to lower $\hat\sigma^2$.

This explains both the appeal and the risk of deleting data to reduce uncertainty, as in F2I. Every observation with a large soft residual raises the fitted noise, including unperturbed observations whose residuals are large by chance, and the two cannot be told apart one observation at a time. A rejection says only that a group of observations is incompatible with the specified model. It does not show that these observations are wrong, and when they reflect real behaviour of the function, revising the model is preferable to deleting them. Appendix~\ref{sec:detectionandforgetting} shows that a greedy deletion rule of our own, which is not the published F2I implementation, becomes overconfident in this way (Tables~\ref{tab:s1} and~\ref{tab:s2}).

\paragraph{Masking.} Closely spaced observations with a common bias predict each other well when each is held out alone. The other members of the group pull each leave-one-out prediction towards the biased value, and every individual residual is small. Holding out the whole group removes this mutual support, and the group residual $z_C$ of (\ref{eq:conditional}) reveals the shared offset. This is the masking problem of classical multiple-outlier diagnostics \cite{beckman1983,hawkins1984,hadi1993}. In the language of (\ref{eq:group}), masking is a cancellation within $S_{CC}$. The individual statistic of observation $i$ is $\alpha_i^2/Q_{ii}$, and under a common offset $c$ on $C$ the mean of $\alpha_i$ is $c\sum_{j\in C}Q_{ij}$. For nearby observations the off-diagonal entries $Q_{ij}$, and hence the inner products $b_i^\top b_j=S_{ij}$ of their spectral coordinates, are typically negative, so this sum is small and the offset largely cancels in every individual statistic. The joint statistic $T_C$ uses the whole block $Q_{CC}$ and does not suffer this cancellation. The remedy is therefore the joint hypothesis, and it is available without any change of coordinates.

\paragraph{Several candidate groups.} With several candidate groups, a Bonferroni or Holm correction \cite{holm1979} controls the probability of any false rejection for a fixed family. Adaptive remove-and-refit searches, such as those of S1--S3, require further care \cite{chen2020}, which is why their calibration is checked empirically below.

\paragraph{Exact and truncated coordinates.} The last form of (\ref{eq:group}) shows that the group test computed in spectral coordinates is identical to the classical multiple-case test. A difference between the two can arise only from approximating the residual matrix, for instance by a hard truncation of the spectrum of $K$. The method ``Group test, Gale coordinates'' below measures what such a truncation costs.

\paragraph{Soft relations as stencils.} For smooth stationary kernels and closely spaced inputs, the eigenvectors of $K$ with small eigenvalues typically oscillate across neighbouring inputs, so the soft relations behave like approximate higher-order finite differences. The first eigenvector below the noise level in Figure~\ref{fig:spectrum} alternates in sign between neighbouring inputs. A row of the residual matrix $S$ is a local stencil: it compares $y_i$ with a weighted average of its neighbours, much like the straight-line relation, and it nearly annihilates constant, linear and quadratic trends.

\paragraph{A classical use of Gale duality.} Properties of a point configuration can be read off from its Gale transform. For example, a subset of the points is the set of points lying on some face of their convex hull exactly when the origin lies in the relative interior of the convex hull of the Gale vectors of the \emph{remaining} points. This makes Gale duality particularly effective for configurations with only a few more points than dimensions, where the dual space is small.

\paragraph{Affine and projective Veronese maps.} The feature map (\ref{eq:veronese}) of all monomials of degree at most $q$ is the affine form of the Veronese map \cite{michalek2021}. The classical Veronese map is projective and uses the monomials of degree exactly $q$ in homogeneous coordinates $(x_0,x_1,\ldots,x_d)$. Setting $x_0=1$ turns these into the monomials of degree at most $q$ in $(x_1,\ldots,x_d)$, so the two descriptions give the same consistency relations.

\section{Supplementary experiments: detection and forgetting}\label{sec:detectionandforgetting}
This appendix reports a simulation study of detection and forgetting, with three supplementary experiments numbered S1--S3.

Experiments S1--S3 turn from testability to detection and forgetting. They check the claims of Section~\ref{sec:noise} that a local discrepancy inflates the fitted noise everywhere and that uncertainty-minimising deletion removes this inflation at the price of overconfidence, and they compare this deletion rule with the test-based rules of Section~\ref{sec:tests}. The three experiments share one data model, pipeline and set of methods, and differ only in the kind of discrepancy: a few shifted observations (S1), a region that the kernel genuinely cannot represent (S2), and a batch of shifted observations that mask one another (S3).
\paragraph{Data.} Experiments S1, S2 and S3 use $n=40$ training inputs in $[-1,1]$, Gaussian measurement noise with variance $\sigma_0^2=0.0025$, the target region $\mathcal A=[0.3,0.9]$, and $200$ held-out test inputs drawn uniformly in $\mathcal A$. Each configuration is repeated for $20$ random seeds.
\begin{itemize}
\item \emph{S1, structured function with a cluster.} The latent function is the cubic $f(x)=0.5+0.8x-1.2x^2+0.9x^3$, with inputs drawn uniformly. The three observations nearest to $x=-0.1$, just outside $\mathcal A$, are shifted by $+0.4$, which is eight noise standard deviations.
\item \emph{S2, genuine model mismatch.} The latent function is $f(x)=\sin3x+1.2\sin(35x)$ for $x<-0.3$ and $f(x)=\sin3x$ for $x\geq-0.3$. On $x<-0.3$ it oscillates far faster than an RBF kernel with a single lengthscale can follow. The roughly $13$ observations with $x<-0.3$ are correct measurements of $f$, but they are counted as misfit observations because the kernel cannot represent them.
\item \emph{S3, a masked batch.} The latent function is the cubic of S1. There are $40-k$ inputs drawn uniformly and a batch of $k$ inputs at $x=0.2\pm0.01$, just outside $\mathcal A$, whose observations are all shifted by $+0.3$. This mimics a mis-calibrated sensor during a run of repeated measurements. We vary $k\in\{2,3,4,5\}$. A variant places the batch inside the target region, at $x=0.5\pm0.01$.
\end{itemize}

\paragraph{Pipeline.} The pipeline is identical for every method, so that differences between methods are attributable to the deletion rule alone. (i) The lengthscale $\ell$, output scale $s$ and noise variance $\sigma^2$ of the RBF kernel $k(x,x')=s\exp(-(x-x')^2/2\ell^2)$ are fitted by maximum a posteriori (MAP) estimation on all $40$ observations, with a log-normal prior $\log\ell\sim\mathcal N(0,1)$ restricted to $\ell\geq0.2$, a flat prior on $\log s$, and a Gamma prior with shape $2$ and rate $0.15$ on $\sigma^2$. The lower bound on $\ell$ expresses the modelling assumption of a smooth latent function. Without it, the fit in S2 selects $\ell\approx0.05$ in most seeds, chases the fast oscillation, and generalises poorly in $\mathcal A$ for every method. (ii) The method returns the observations to keep. (iii) The noise variance alone is refitted by MAP on the kept observations, with the kernel fixed. (iv) The GP predicts at the test inputs.

\paragraph{Methods.} The tables compare six methods, named as in the table rows. The first two are references that bracket the comparison, and the other four are deletion rules.
\begin{itemize}
\item \emph{Full GP.} Keep every observation. This shows the effect of the misfit observations when nothing is done about them.
\item \emph{Oracle.} Delete exactly the planted misfit observations and nothing else. No real method can know which observations these are, so the oracle is a best-case reference rather than a competitor.
\item \emph{F2I-style.} Greedy deletion on the objective (\ref{eq:f2i}), the total predictive variance over a grid of $30$ inputs in $\mathcal A$. At every step each single deletion is evaluated with an exact refit of the noise variance, the best one is applied, and the search stops when no deletion lowers (\ref{eq:f2i}). This uses the published objective \cite{wendl2026} with exact refits instead of first-order influence approximations, and is not the authors' implementation.
\item \emph{Single-observation test.} Repeatedly delete the observation with the largest standardised leave-one-out residual $|\alpha_i|/\sqrt{Q_{ii}}$, recomputed after every deletion, while it exceeds the Bonferroni threshold $z_{1-0.05/(2n)}\approx3.2$, the corresponding quantile of the standard normal distribution.
\item \emph{Group test, classical.} The test (\ref{eq:group}) applied greedily. Candidate groups are all single observations and all runs of $2$ to $s_{\max}$ neighbouring inputs, with $s_{\max}=3$ in S1 and S2 and $s_{\max}=5$ in S3. At every step the candidate with the smallest $p$-value under (\ref{eq:chisq}) is deleted, and the search stops when that $p$-value exceeds $0.05/M$, where $M$ is the number of remaining candidates.
\item \emph{Group test, Gale coordinates.} The same candidates, search and stopping rule, but with the test computed in a Gale basis. The basis $\tilde G$ consists of the eigenvectors of $K$ whose eigenvalues do not exceed the reference noise, a hard-threshold approximation of the soft construction of Section~\ref{sec:gp}. With $\tilde P_N=\tilde G\tilde G^\top$ and $\tilde e=\tilde P_Ny$, a group $C$ is scored by the increase in explained test failure $\tilde e_C^\top((\tilde P_N)_{CC})^+\tilde e_C$, which is the hard-truncated Gale form of (\ref{eq:group}). With the exact soft operator in place of $\tilde P_N$, this method would coincide with the classical group test (Section~\ref{sec:tests}). It is included to measure what the truncation costs.
\end{itemize}
All three tests use the same reference noise variance: the MAP noise of the GP restricted to the training observations inside $\mathcal A$. This avoids judging failures against a noise estimate inflated by the very discrepancy under test (Section~\ref{sec:tests}), but it assumes that the model is adequate inside $\mathcal A$.

\paragraph{Metrics.} On the test inputs in $\mathcal A$ we report the mean squared error (MSE) of the predictive mean against the latent function, the empirical coverage of the $95\%$ predictive intervals for new noisy observations (ideally $0.95$), the negative log predictive density (NLPD) of those observations, and the mean predictive variance (sharpness), which is the F2I objective divided by the number of target inputs. On the training set we report the fraction of misfit observations deleted (higher is better), the fraction of honest observations deleted (lower is better), and the refitted noise variance $\hat\sigma^2$, whose true value is $\sigma_0^2=0.0025$. All numbers are produced by the code accompanying this paper.

\paragraph{Noise inflation and overconfidence.} Tables~\ref{tab:s1} and~\ref{tab:s2} confirm the mechanism of Section~\ref{sec:noise}. In S1, three shifted observations raise the fitted noise variance of the full GP to $9.8\times10^{-3}$, almost four times the true value. In S2 the rough region raises it to $0.22$, about ninety times the true value. Because the same noise enters every prediction, the intervals in the target region are far too wide, and their coverage is $1.00$ instead of $0.95$.

F2I-style deletion removes this inflation but does not stop there. It keeps only about $12$ of the $40$ observations, deletes roughly two thirds of the honest ones, and drives the fitted noise variance to about $10^{-5}$, two orders of magnitude below the truth. Its predictions are the sharpest of all methods, yet only $9$ to $14\%$ of new observations fall inside its nominal $95\%$ intervals, and its NLPD exceeds $100$, against values between $-1.5$ and $0.2$ for every other method. Figure~\ref{fig:calibration} shows this trade-off for all three experiments: F2I-style deletion wins on sharpness, the criterion it optimises, and loses on calibration. Figure~\ref{fig:examples} shows typical fits. The test-based methods delete at most $3.2\%$ of the honest observations, recover a noise variance close to the truth (with one exception, discussed below), and keep the coverage between $0.93$ and $0.97$.

S2 is the hardest case for every method. The rough region contains stretches where the fast oscillation happens to look smooth, so even the best test finds only about $80\%$ of the misfit observations. Nevertheless the remaining misfit has little influence on the target region, and the accuracy of every test-based method there is within $15\%$ of the oracle.

\begin{table}[!htbp]
\centering
\small
\resizebox{\linewidth}{!}{\begin{tabular}{lrrrrrrr}
\toprule
Method & Kept & $\hat\sigma^2$ ($\times10^{-3}$) & MSE ($\times10^{-4}$) $\downarrow$ & Coverage & NLPD $\downarrow$ & Misfit del. $\uparrow$ & Honest del. $\downarrow$\\
\midrule
Full GP & 40.0 & 9.78 & 8.8 $\pm$ 10.2 & 1.00 $\pm$ 0.01 & -1.18 & 0.00 & 0.000\\
Oracle & 37.0 & 2.60 & 5.6 $\pm$ 6.5 & 0.94 $\pm$ 0.03 & -1.46 & 1.00 & 0.000\\
F2I-style & 12.2 & 0.01 & 10.2 $\pm$ 6.7 & 0.12 $\pm$ 0.06 & 154.46 & 0.85 & 0.681\\
Single-observation test & 37.5 & 3.11 & 7.0 $\pm$ 8.6 & \textbf{0.94 $\pm$ 0.05} & -1.43 & 0.65 & 0.014\\
Group test, classical & 37.0 & \textbf{2.54} & \textbf{5.6 $\pm$ 6.5} & \textbf{0.94 $\pm$ 0.04} & \textbf{-1.45} & \textbf{0.95} & \textbf{0.003}\\
Group test, Gale coordinates & 37.2 & 2.66 & 6.6 $\pm$ 8.7 & \textbf{0.94 $\pm$ 0.05} & \textbf{-1.45} & 0.80 & 0.011\\
\bottomrule
\end{tabular}}
\caption{S1, a cubic function with a cluster of three shifted observations just outside the target region. Means over 20 seeds, $\pm$ one standard deviation across seeds. The true noise variance is $\sigma_0^2=2.5\times10^{-3}$. Rows: \emph{Full GP} keeps all observations. \emph{Oracle} deletes exactly the planted misfit observations, which no real method can know, and serves as a best-case reference. \emph{F2I-style} deletes greedily to minimise predictive variance in $\mathcal A$. The remaining three rows are the single-observation test and the group test in classical and in Gale coordinates. Columns: number of observations kept, refitted noise variance, mean squared error and NLPD in $\mathcal A$, coverage of the $95\%$ predictive intervals, and the fractions of misfit and of honest observations deleted. Arrows indicate whether lower or higher values are better. Bold marks the best deletion rule in each column, where best means closest to the true noise variance and to the nominal coverage $0.95$. The Full GP and Oracle rows are references and are not marked.}\label{tab:s1}
\end{table}

\begin{table}[!htbp]
\centering
\small
\resizebox{\linewidth}{!}{\begin{tabular}{lrrrrrrr}
\toprule
Method & Kept & $\hat\sigma^2$ ($\times10^{-3}$) & MSE ($\times10^{-4}$) $\downarrow$ & Coverage & NLPD $\downarrow$ & Misfit del. $\uparrow$ & Honest del. $\downarrow$\\
\midrule
Full GP & 40.0 & 222.49 & 53.6 $\pm$ 49.4 & 1.00 $\pm$ 0.00 & 0.21 & 0.00 & 0.000\\
Oracle & 26.7 & 2.45 & 5.5 $\pm$ 6.0 & 0.93 $\pm$ 0.05 & -1.45 & 1.00 & 0.000\\
F2I-style & 11.9 & 0.03 & 15.3 $\pm$ 9.3 & 0.14 $\pm$ 0.08 & 143.86 & 0.80 & 0.648\\
Single-observation test & 30.7 & 3.77 & 6.3 $\pm$ 5.8 & \textbf{0.95 $\pm$ 0.05} & -1.41 & 0.64 & 0.032\\
Group test, classical & 28.5 & \textbf{2.86} & 5.7 $\pm$ 6.2 & 0.94 $\pm$ 0.05 & \textbf{-1.43} & \textbf{0.81} & 0.027\\
Group test, Gale coordinates & 29.6 & 10.44 & \textbf{5.6 $\pm$ 4.2} & 0.97 $\pm$ 0.04 & -1.22 & 0.75 & \textbf{0.015}\\
\bottomrule
\end{tabular}}
\caption{S2, a function that oscillates faster than the kernel can follow for $x<-0.3$. Same layout as Table~\ref{tab:s1}.}\label{tab:s2}
\end{table}

\begin{figure}[!htbp]
\centering
\includegraphics[width=\linewidth]{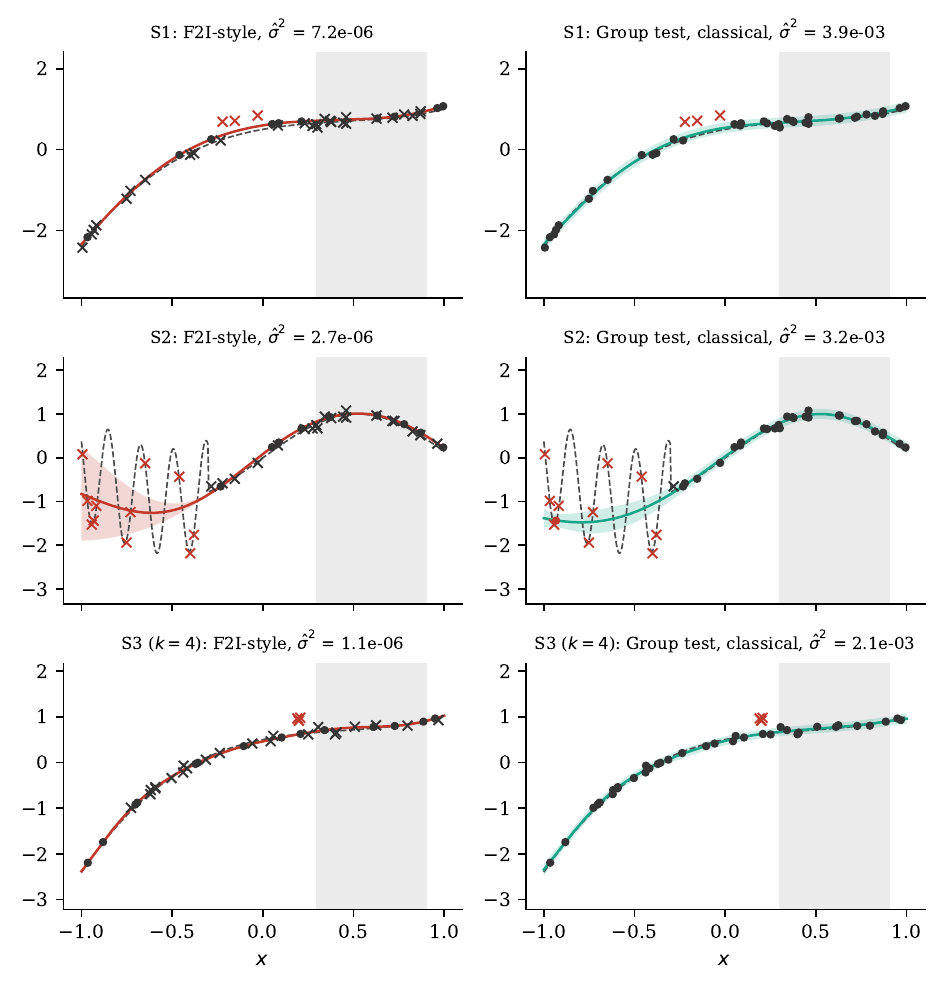}
\caption{One seed of each experiment. Left: F2I-style deletion. Right: the classical group test. Dots are kept observations and crosses deleted ones. Red marks the planted misfit observations, and the dashed curve is the latent function. Lines and bands show the predictive mean and the $95\%$ predictive interval for new observations. The target region $\mathcal A$ is shaded. F2I-style deletion removes most honest observations and leaves intervals that are far too narrow. The group test removes the misfit observations and little else.}\label{fig:examples}
\end{figure}

\begin{figure}[!htbp]
\centering
\includegraphics[width=0.75\linewidth]{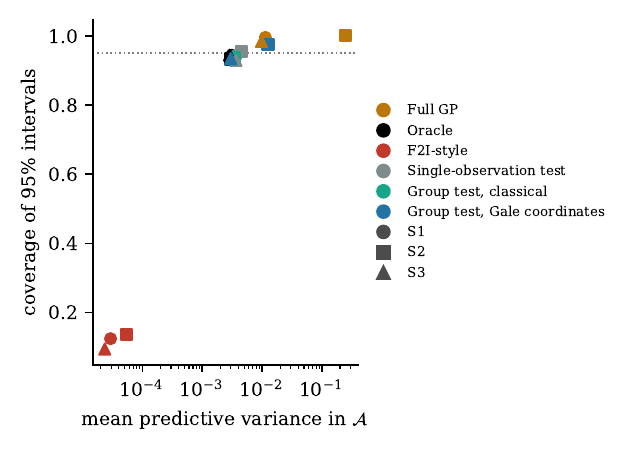}
\caption{Sharpness against calibration in S1, S2 and S3 ($k=4$). Each marker is the mean over 20 seeds for one method and one experiment. The dotted line marks the nominal coverage of $0.95$. A small predictive variance is only useful if the coverage remains close to this line.}\label{fig:calibration}
\end{figure}

\paragraph{Deletion trajectories.} Figure~\ref{fig:trajectories} follows F2I-style deletion one step at a time, refitting the noise and evaluating the predictions in $\mathcal A$ after every single deletion, in order to see where it goes wrong. Each column is one experiment, and the horizontal axis counts the observations deleted so far. The top row shows, for each step, the fraction of seeds in which the observation deleted at that step is a planted misfit one, so a bar near $1$ means that a misfit observation was removed and a bar near $0$ that an honest one was. The middle row shows the refitted noise variance relative to the truth, where $1$ is correct, and the bottom row the coverage of the $95\%$ predictive intervals in $\mathcal A$, where $0.95$ is ideal. Thin lines are individual seeds and thick lines their median. The dashed vertical line marks the number of planted misfit observations, which is where a method that deleted exactly the misfit observations would stop. The green lines follow the classical group test in the same way, one selected group at a time in its own order of selection, and markers show where each method stops in each seed, red for F2I-style deletion and green for the group test, so that its path and stopping point can be compared with those of F2I-style deletion and with the ideal stopping point. The group test stops after a few deletions, mostly of misfit observations, so its noise variance ends close to the truth and its coverage close to $0.95$. The early deletions mostly hit misfit observations: about $70\%$ of the first deletions in each experiment are misfit ones. When as many observations have been deleted as there are misfit ones, the median fitted noise variance in S1 and S3 is within $2\%$ of the true value, and the median coverage is $0.94$. At that point F2I-style deletion is essentially as good as the oracle. It does not stop there, however, because every further deletion still lowers its objective. Removing an honest observation whose residual is large by chance lowers the refitted noise, and hence the predictive variance in $\mathcal A$, just as removing a misfit one does. F2I-style deletion stops only when so few observations remain that losing another raises the predictive variance more than the smaller noise lowers it. The group test stops for a different reason: it deletes a group only if that group is statistically surprising under the reference noise, and once the misfit observations are gone no remaining group passes the Bonferroni threshold. One rule stops when deleting no longer pays, the other when nothing left is surprising. It goes on to delete between $17$ and $25$ honest observations on average, the fitted noise falls two to three orders of magnitude below the truth, and coverage collapses. The group test stops close to the point where F2I-style deletion should have stopped (green markers). The failure of F2I-style deletion is therefore not in what it deletes first, but in the absence of a stopping rule that distinguishes a misfit observation from an honest one with a large residual. A test-based stopping rule supplies such a rule, approximately calibrated, as the next paragraph shows.

\paragraph{Calibration of the complete pipeline.} Equation~(\ref{eq:chisq}) is exact only for a fixed group and a specified covariance, whereas S1--S3 estimate the kernel from the same data, estimate the reference noise, and select and remove groups adaptively. Table~\ref{tab:calibration} therefore runs the complete pipeline on $1000$ data sets per configuration that contain no misfit observations, generated from the cubic of S1 with uniform inputs or with the cluster of S3. Any deletion is then a false rejection. The pipeline is moderately liberal. At least one observation is deleted in $6.7$ to $7.6\%$ of the data sets by the group tests and in $8.8$ to $9.4\%$ by the single-observation test, against a nominal $5\%$, and on average fewer than $1\%$ of the observations are deleted. The data are generated from a fixed cubic rather than drawn from the GP, so this benchmark does not isolate the calibration of the test under its own null hypothesis. The excess over the nominal level may reflect both the plug-in of estimated hyperparameters \cite{bachoc2013} and the mismatch between the cubic and the GP prior. Table~\ref{tab:calibrationgp} separates the two effects with data drawn from the GP itself, $1000$ data sets per design. With the true kernel and noise variance, the single-observation test deletes something in $4.8$ to $5.3\%$ of the data sets, close to the nominal $5\%$, and the group tests in $1.5$ to $3.2\%$, because the Bonferroni correction over overlapping candidate groups is conservative. With hyperparameters estimated from the same data, the rates rise to $6.3$ to $8.5\%$. The excess over the nominal level in Table~\ref{tab:calibration} is therefore mainly caused by estimating the hyperparameters, not by the mismatch between the cubic and the GP prior. It is small compared with the deletions of F2I-style selection, but the stopping rule should be read as approximately rather than exactly calibrated.

\begin{table}[!htbp]
\centering
\small
\begin{tabular}{lrrrr}
\toprule
 & \multicolumn{2}{c}{Uniform inputs} & \multicolumn{2}{c}{With a cluster}\\
Method & Any deletion $\downarrow$ & Deleted $\downarrow$ & Any deletion $\downarrow$ & Deleted $\downarrow$\\
\midrule
Single-observation test & 0.088 & \textbf{0.0036} & 0.094 & \textbf{0.0044}\\
Group test, classical & \textbf{0.073} & 0.0059 & \textbf{0.067} & 0.0082\\
Group test, Gale coordinates & 0.076 & 0.0060 & 0.072 & 0.0087\\
\bottomrule
\end{tabular}
\caption{Calibration of the complete deletion pipeline of S1--S3 when the data contain no misfit observations, over 1000 seeds per configuration. ``Any deletion'' is the fraction of seeds in which at least one observation is deleted, which the Bonferroni stopping rule is meant to keep near $0.05$. ``Deleted'' is the mean fraction of the 40 observations deleted. Bold marks the best value in each column.}\label{tab:calibration}
\end{table}

\begin{table}[!htbp]
\centering
\small
\begin{tabular}{lrrrr}
\toprule
 & \multicolumn{2}{c}{Uniform inputs} & \multicolumn{2}{c}{With a cluster}\\
Method & Fixed & Estimated & Fixed & Estimated\\
\midrule
Single-observation test & 0.053 & 0.084 & 0.048 & 0.085\\
Group test, classical & 0.025 & 0.072 & 0.015 & 0.063\\
Group test, Gale coordinates & 0.032 & 0.076 & 0.023 & 0.067\\
\bottomrule
\end{tabular}
\caption{Calibration under a true GP null: fraction of 1000 data sets drawn from the GP in which at least one observation is deleted, against a nominal $0.05$. ``Fixed'' uses the true kernel and noise variance, ``Estimated'' the pipeline of S1--S3 with hyperparameters fitted to the same data.}\label{tab:calibrationgp}
\end{table}

\begin{figure}[!htbp]
\centering
\includegraphics[width=\linewidth]{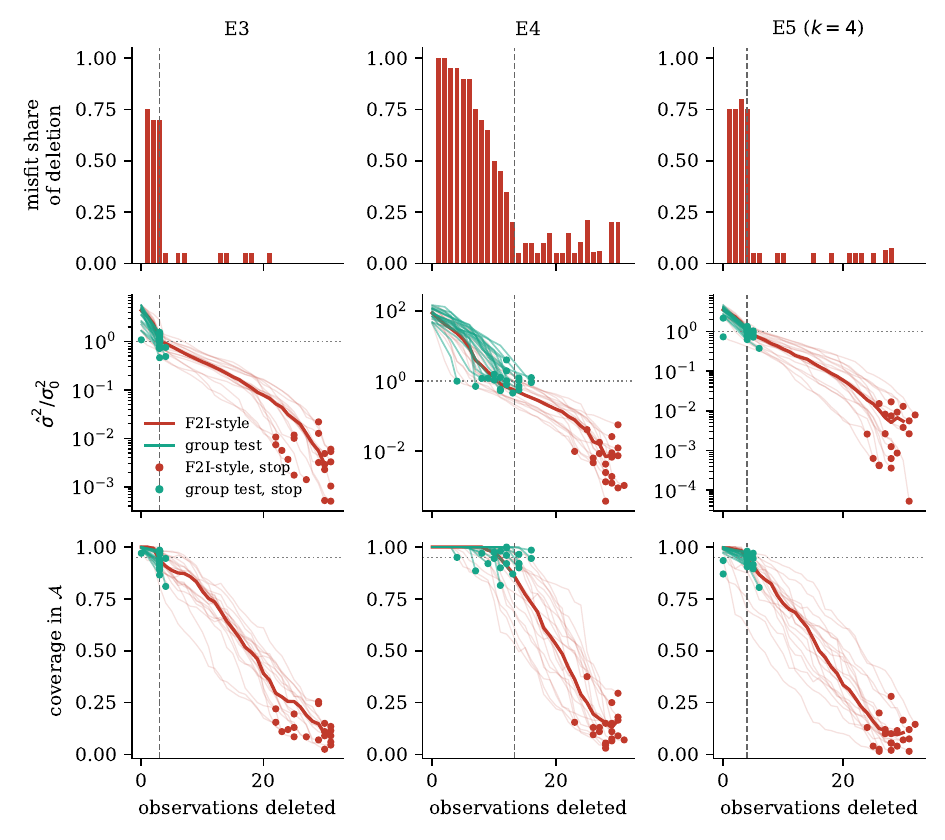}
\caption{Deletion trajectories of F2I-style deletion, from no deletion until its stopping point. Top: fraction of seeds in which the observation deleted at that step is a misfit one. Middle: refitted noise variance relative to the truth (logarithmic scale). Bottom: coverage of the $95\%$ predictive intervals in $\mathcal A$. Thin lines are individual seeds and thick lines the median. The dashed vertical line marks the number of misfit observations, and the dotted horizontal lines mark the ideal values. Green lines follow the classical group test one selected group at a time. Markers show where each method stops in each seed, red for F2I-style deletion and green for the group test.}\label{fig:trajectories}
\end{figure}

\paragraph{Masking.} Table~\ref{tab:s3} and Figure~\ref{fig:masking} show the masked batch of S3. Each observation in the batch is predicted well by its shifted neighbours when it is held out alone, so the single-observation test finds only $68\%$ of a batch of four, and its error in the target region is $2.4$ times that of the group tests. The gap widens with the batch size. At $k=5$ the single-observation test finds half of the batch, while the group tests still find $78\%$. When the batch lies inside the target region, every test fails to find it, and the error in $\mathcal A$ is about twenty times that of the oracle. The reason is the reference noise. It is estimated from the observations in $\mathcal A$, which now include the batch, and its mean of $2.2\times10^{-2}$ is about six times the value in the other S3 runs. Against this inflated reference nothing looks surprising. This is a genuine limitation of a target-region reference, discussed in Section~\ref{sec:tests}.

\begin{table}[!htbp]
\centering
\small
\resizebox{\linewidth}{!}{\begin{tabular}{lrrrrrrr}
\toprule
Method & Kept & $\hat\sigma^2$ ($\times10^{-3}$) & MSE ($\times10^{-4}$) $\downarrow$ & Coverage & NLPD $\downarrow$ & Misfit del. $\uparrow$ & Honest del. $\downarrow$\\
\midrule
Full GP & 40.0 & 8.33 & 29.3 $\pm$ 12.7 & 0.98 $\pm$ 0.03 & -1.11 & 0.00 & 0.000\\
Oracle & 36.0 & 2.31 & 3.9 $\pm$ 2.6 & 0.94 $\pm$ 0.02 & -1.50 & 1.00 & 0.000\\
F2I-style & 11.8 & 0.01 & 24.2 $\pm$ 34.2 & 0.09 $\pm$ 0.07 & 649.65 & \textbf{0.91} & 0.682\\
Single-observation test & 36.8 & 3.03 & 17.8 $\pm$ 24.3 & \textbf{0.93 $\pm$ 0.05} & -1.32 & 0.68 & 0.014\\
Group test, classical & 36.1 & \textbf{2.47} & \textbf{7.3 $\pm$ 11.2} & \textbf{0.93 $\pm$ 0.04} & \textbf{-1.45} & 0.90 & \textbf{0.008}\\
Group test, Gale coordinates & 36.1 & \textbf{2.47} & \textbf{7.3 $\pm$ 11.2} & \textbf{0.93 $\pm$ 0.04} & \textbf{-1.45} & 0.90 & \textbf{0.008}\\
\bottomrule
\end{tabular}}
\caption{S3, a masked batch of $k=4$ shifted observations just outside the target region. Same layout as Table~\ref{tab:s1}.}\label{tab:s3}
\end{table}

\begin{figure}[!htbp]
\centering
\includegraphics[width=\linewidth]{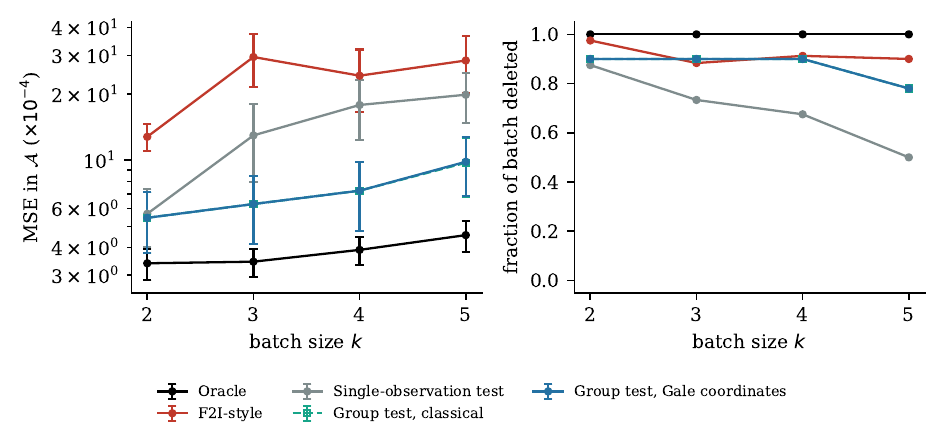}
\caption{S3, a masked batch of $k$ shifted observations. Left: mean squared error in the target region, with error bars of one standard error over 20 seeds (logarithmic scale). Right: fraction of the batch deleted. The two group tests coincide.}\label{fig:masking}
\end{figure}

\paragraph{Gale coordinates against the classical group test.} Table~\ref{tab:agreement} compares the two group tests directly. In S3 they make identical decisions in $99$ of the $100$ runs, including the run with the batch inside $\mathcal A$. The gain over single-observation tests in S3 is therefore entirely a gain of testing groups, which the classical test delivers without any Gale vocabulary. In S1 and S2 the two tests differ more often, because the Gale version uses a hard spectral threshold that approximates the soft operator $S$ of Section~\ref{sec:gp}. These differences never favour the Gale version: in S1 it finds fewer misfit observations ($80\%$ against $95\%$), and in S2 it leaves a noise variance nearly four times larger, with equal accuracy. These results quantify what the hard truncation costs. Together with the identity of Section~\ref{sec:tests}, they show that, for data already collected, the improvement comes from testing the right group jointly, not from the coordinates in which the test is expressed. The Gale view does not detect more discrepancies in such data. Its contribution is the prospective one shown in E1 to E3: before data are collected, it says which assumptions a design can test and with what power.

\begin{table}[!htbp]
\centering
\small
\resizebox{\linewidth}{!}{\begin{tabular}{lrrrr}
\toprule
Experiment & Identical decisions & MSE, single ($\times10^{-4}$) $\downarrow$ & MSE, classical ($\times10^{-4}$) $\downarrow$ & MSE, Gale ($\times10^{-4}$) $\downarrow$\\
\midrule
S1 & 17/20 & 7.0 & 5.6 & 6.6\\
S2 & 3/20 & 6.3 & 5.7 & 5.6\\
S3, $k=2$ & 20/20 & 5.7 & 5.5 & 5.5\\
S3, $k=3$ & 20/20 & 13.0 & 6.3 & 6.3\\
S3, $k=4$ & 20/20 & 17.8 & \textbf{7.3} & \textbf{7.3}\\
S3, $k=5$ & 19/20 & 19.8 & \textbf{9.7} & \textbf{9.8}\\
S3, $k=4$, batch inside $\mathcal A$ & 20/20 & 99.6 & 99.6 & 99.6\\
\bottomrule
\end{tabular}}
\caption{Group tests in Gale coordinates against classical group tests with the same candidate groups, reference noise and stopping rule. ``Identical decisions'' counts the seeds in which both delete exactly the same observations. The single-observation test is shown for comparison. Means over 20 seeds. Bold marks a method whose MSE is smaller than that of at least one other method by more than two paired standard errors over the seeds, and not significantly larger than that of any. Arrows indicate that lower values are better.}\label{tab:agreement}
\end{table}

\sloppy

\end{document}